\documentclass{article}
\usepackage[preprint]{neurips_2025}

\makeatletter
\renewcommand{\@notice}{}
\makeatother
\usepackage[utf8]{inputenc}
\usepackage[T1]{fontenc}
\usepackage{microtype}
\usepackage{amsmath,amssymb,amsfonts,amsthm}
\usepackage{graphicx}
\usepackage{booktabs}
\usepackage{enumitem}
\usepackage{natbib}
\usepackage{float}
\usepackage{algorithm}
\usepackage{algorithmic} 
\setcitestyle{authoryear,open={(},close={)}}
\usepackage{xcolor}
\usepackage[colorlinks=true,citecolor=blue,linkcolor=black]{hyperref}

\newtheorem{proposition}{Proposition}
\newtheorem{definition}{Definition}
\newtheorem{theorem}{Theorem}
\newtheorem{assumption}{Assumption}
\newtheorem{remark}{Remark}

\title{Embedded Safety-Aligned Intelligence via Differentiable Internal Alignment Embeddings}

\author{
  \textbf{Harsh Rathva}, \textbf{Ojas Srivastava}, \textbf{Pruthwik Mishra} \\
  Sardar Vallabhbhai National Institute of Technology (SVNIT), Surat, India \\
  \small{\texttt{\{u24ai036, u24ai028, pruthwikmishra\}@aid.svnit.ac.in}}
}

\begin{document}

\maketitle

\begin{abstract}
We introduce \emph{Embedded Safety-Aligned Intelligence} (ESAI), a theoretical framework for multi-agent reinforcement learning that embeds alignment constraints directly into agents' internal representations via differentiable \emph{internal alignment embeddings} (IAE). Unlike external reward shaping or post-hoc safety constraints, IAE are learned latent variables that predict externalized harm through counterfactual reasoning and modulate policy gradients toward harm reduction via attention gating and graph diffusion.

We formalize the ESAI framework through four integrated mechanisms: (1) differentiable counterfactual alignment penalties computed via softmin reference distributions, (2) IAE-weighted attention biasing perceptual salience toward alignment-relevant features, (3) Hebbian affect-memory coupling supporting temporal credit assignment, and (4) similarity-weighted graph diffusion with bias-mitigation controls. We derive conditions for bounded internal embeddings under Lipschitz constraints and spectral radius bounds, analyze computational complexity as $O(Nkd)$ for $N$ agents with $k$-dimensional embeddings, and discuss theoretical properties including contraction dynamics and fairness-performance tradeoffs.

This work positions ESAI as a conceptual contribution to differentiable alignment mechanisms in multi-agent systems. We identify open theoretical questions regarding convergence guarantees, optimal embedding dimensionality, and extension to high-dimensional state spaces. Empirical validation remains future work.
\end{abstract}

\section{Introduction}
\label{sec:intro}

Contemporary multi-agent reinforcement learning (MARL) optimizes explicit task objectives but typically lacks internal differentiable regulators that encourage prosocial behavior and stable coordination under distribution shift. Standard approaches to alignment—such as reward shaping \citep{ng1999policy}, constrained optimization \citep{achiam2017constrained}, or inverse RL \citep{hadfield2016cooperative}—rely on external supervision signals that are either hand-designed, require human preference data, or operate as non-differentiable constraints decoupled from policy learning dynamics.

We investigate whether agents can instead learn an \emph{internal alignment embedding} (IAE): a differentiable latent variable that tracks predicted externalized harm and shapes policy gradients toward harm reduction through three key properties:
\begin{enumerate}[leftmargin=*]
  \item \textbf{Predictive}: IAE forecasts alignment-relevant outcomes via counterfactual reasoning over candidate actions.
  \item \textbf{Regulatory}: IAE modulates perception and action selection through attention gating and gradient coupling.
  \item \textbf{Distributed}: IAE propagates across agent neighborhoods via graph diffusion with controllable similarity weighting.
\end{enumerate}

This paper presents \emph{Embedded Safety-Aligned Intelligence} (ESAI), a theoretical framework formalizing these principles. ESAI differs from existing alignment approaches in three fundamental ways: (1) alignment constraints are \emph{embedded} in learned internal representations rather than imposed externally, (2) harm prediction is \emph{differentiable} and jointly trained with policy parameters, and (3) multi-agent coordination emerges from \emph{graph diffusion} of alignment signals rather than centralized control or explicit communication protocols.

\paragraph{Theoretical contributions.}
We formalize ESAI through:
\begin{itemize}[leftmargin=*]
  \item A differentiable counterfactual alignment penalty that enables gradient-based harm forecasting without discrete argmin operations (Sec.~\ref{sec:counterfactual}).
  \item IAE-weighted attention mechanisms that bias perceptual salience toward alignment-relevant features (Sec.~\ref{sec:attention}).
  \item Hebbian affect-memory coupling with differentiable read operations supporting temporal credit assignment (Sec.~\ref{sec:hebbian}).
  \item Locally damped graph diffusion with similarity-weighted propagation and bias-mitigation controls (Sec.~\ref{sec:diffusion}).
  \item Conditions for bounded internal embeddings under Lipschitz constraints and spectral radius bounds (Sec.~\ref{sec:stability}).
\end{itemize}

\paragraph{Scope and positioning.}
ESAI is a \emph{conceptual framework}, not a validated algorithm. We present one possible architectural instantiation (Sec.~\ref{sec:arch}) but emphasize that alternative implementations may satisfy ESAI principles. Key open questions include optimal embedding dimensionality, convergence guarantees under stochastic dynamics, and scalability to high-dimensional observation spaces. \textbf{Empirical validation across diverse environments remains necessary future work} to assess whether ESAI's theoretical advantages translate to practical coordination improvements.

IAE is a computational abstraction—\emph{not} a claim about subjective emotion or consciousness. All theoretical analysis assumes availability of domain-specific harm metrics; we discuss implications of this assumption and potential biases in Sec.~\ref{sec:limitations}.

\section{Related Work}
\label{sec:related}

\subsection{Alignment and Safety in MARL}

Cooperative inverse reinforcement learning \citep{hadfield2016cooperative} infers alignment objectives from human demonstrations but requires extensive preference data and does not maintain differentiable internal alignment states. Deep reinforcement learning from human preferences \citep{christiano2017deep} learns reward models from comparisons but operates on scalar rewards rather than internal representations that can gate perception or modulate credit assignment.

Constrained policy optimization \citep{achiam2017constrained} enforces safety via Lagrangian constraints but lacks differentiable internal alignment states that persist across timesteps. \citet{alshiekh2018safe} introduce safe RL via shielding, which uses formal methods to prevent unsafe actions but operates as an external constraint layer rather than an embedded internal mechanism.

\citet{hughes2018inequity} demonstrate that intrinsic inequity aversion—an internal social preference—improves cooperation in multi-agent settings. While this shows that intrinsic motivations can enhance coordination, the preference structure is hand-designed rather than learned from predicted outcomes.

Recent work explores complementary safety mechanisms. \citet{chatterji2025smarl} introduce SMARL with probabilistic logic shields that reduce safety violations through external symbolic constraints operating separately from policy learning. \citet{mayoral2025designing} investigate designing ethical environments using MARL, demonstrating that environmental structure can encourage prosocial behavior by shaping the task itself rather than agent internals.

Multi-objective RL \citep{hayes2022practical} optimizes Pareto frontiers over competing objectives, providing a comprehensive framework for balancing multiple goals but not integrating alignment-specific internal representations or mechanisms for distributed coordination.

\textbf{ESAI contributes} a differentiable framework where alignment emerges from learned internal embeddings that predict harm via counterfactual reasoning and modulate policy gradients through attention gating, rather than relying on external rewards, hand-designed preferences, discrete structures, or non-differentiable constraints.

\subsection{Counterfactual Reasoning in MARL}

Counterfactual Multi-Agent Policy Gradients (COMA) \citep{foerster2018counterfactual} computes counterfactual baselines by marginalizing single-agent actions for centralized credit assignment. These counterfactuals reduce variance in gradient estimates but do not predict or prevent externalized harm.

\citet{zhou2022pac} introduce Prediction-Assisted Counterfactual (PAC) methods that use counterfactual predictions as auxiliary losses to assist value factorization in MARL. PAC demonstrates that learned counterfactual models can improve representational quality, but targets credit assignment rather than alignment.

\citet{yao2024cmix} develop CMIX, applying causal value decomposition for cooperative MARL. This work provides structured causal inference for multi-agent credit assignment but does not address safety or harm prediction.

\citet{alamiyan2023kindmarl} introduce a framework for measuring kindness in multi-agent reinforcement learning. \textbf{This work is conceptually closest to ESAI} in exploring intrinsic prosocial mechanisms. However, it focuses on defining and measuring kindness metrics rather than learning differentiable internal alignment states that forecast harm through counterfactual reasoning and gradient flow.

\textbf{ESAI extends} counterfactual reasoning from credit assignment to differentiable alignment penalties, training IAE to forecast harm and penalizing deviations from harm-minimizing reference policies via softmin distributions that preserve gradient flow for joint policy-predictor learning.

\subsection{Zero-Shot Scaling and Population Generalization}

A long-standing challenge in MARL is transferring policies trained on small agent populations to larger groups without retraining. \citet{carion2019structured} use structured prediction approaches for generalization in cooperative MARL, demonstrating that learned models can be reused on problems with different numbers of agents via optimization-based inference.

\citet{guo2025shppo} propose Scalable and Heterogeneous PPO (SHPPO) for multi-agent learning, integrating parameter sharing with heterogeneous components for improved scalability. SHPPO demonstrates robustness across population sizes through a balance of shared and agent-specific parameters.

The literature identifies key architectural features for robust scaling: permutation invariance, local message passing via graph neural networks, composable action representations, and parameter sharing. However, existing work does not explicitly model internal alignment states that propagate via graph diffusion.

\textbf{ESAI incorporates} these scaling principles (permutation invariance via symmetric IAE updates, local message passing via graph diffusion) while adding alignment-specific mechanisms. The graph diffusion operator enables decentralized IAE propagation that could support population scaling, though empirical validation is needed to assess whether alignment pressure degrades less than task performance under scaling.

\subsection{Graph Neural Networks and Attention Mechanisms in MARL}

Graph neural networks enable structured communication in MARL by modeling agents as nodes with learned edge weights. \citet{jiang2020graph} provide foundational work on graph convolutional RL, demonstrating that GNN architectures can capture relational structure in multi-agent problems.

Attention mechanisms in MARL compute dynamic edge weights and gate messages so agents learn which peers are relevant in each state. \citet{iqbal2019actor} introduce Actor-Attention-Critic for multi-agent learning, showing that attention over other agents improves coordination. \citet{liu2020g2anet} develop G2ANet using graph attention mechanisms for multi-agent game abstraction, demonstrating improved performance through learned relational representations. \citet{wang2021qplex} propose QPLEX with duplex dueling and attention mechanisms for value factorization, demonstrating improved stability. \citet{liu2023na2q} develop NA2Q using neural attention additive models for interpretable multi-agent Q-learning, showing that attention can provide both performance and interpretability.

However, existing work does not integrate graph mechanisms with internal alignment states that track predicted harm.

\textbf{ESAI integrates} graph diffusion with IAE dynamics, enabling similarity-weighted propagation of alignment-relevant information. The diffusion operator propagates IAE across neighborhoods, while similarity weighting modulates edge strengths based on learned agent identities. The bias-mitigation regularizer provides interpretable fairness-performance tradeoffs not present in standard graph communication architectures.

\subsection{Memory and Plasticity in MARL}

Differentiable memory systems such as Neural Turing Machines \citep{graves2014neural} enable learned read/write operations for temporal credit assignment, showing that memory mechanisms can be trained end-to-end via gradient descent.

\citet{miconi2018differentiable} demonstrate that backpropagation can train differentiable Hebbian plasticity rules that adapt network weights during deployment, showing benefits for continual learning and adaptation. This work proves that associative learning rules can be optimized jointly with network parameters.

However, no existing work couples Hebbian learning with alignment-specific internal states or uses Hebbian traces to support counterfactual forecasting of harm. Standard memory architectures (attention over episodic buffers, recurrent states) do not encode associative traces between internal alignment embeddings and perceptual features.

\textbf{ESAI couples} Hebbian traces to IAE dynamics via differentiable read operations, enabling alignment-aware memory updates that support counterfactual forecasting. The Hebbian matrix encodes co-activation history of IAE and percepts, providing context for harm prediction while preserving end-to-end trainability.

\subsection{Vision-Language Models for Policy Guidance}

Recent work explores using vision-language models to provide semantic guidance for reinforcement learning. \citet{wu2025varl} propose using VLMs as action advisors for online RL, demonstrating that pre-trained models can provide interpretable action suggestions. However, this approach requires external model queries at each timestep rather than learning intrinsic alignment representations embedded in the agent's parameters.

\textbf{ESAI differs} by learning alignment representations end-to-end without requiring external model queries, pre-trained components, or oracle access to semantic models during deployment.

\subsection{Relation to Potential-Based Reward Shaping}

Potential-based reward shaping \citep{ng1999policy} adds $\Phi(s_{t+1}) - \Phi(s_t)$ to rewards, preserving optimal policies under certain conditions and providing theoretical guarantees.

ESAI differs in three fundamental ways:
\begin{enumerate}[itemsep=2pt]
  \item \textbf{Learned dynamics}: IAE evolves via learned function $g_\phi$ and graph diffusion operator $L$, not hand-designed potential functions. This enables adaptation to environment-specific harm structures without manual redesign.
  \item \textbf{Counterfactual supervision}: IAE is trained to predict harm outcomes via counterfactual forecasting, not approximate task value. This separates alignment objectives from task rewards.
  \item \textbf{Perceptual gating}: IAE modulates perception via attention, enabling non-Markovian salience modulation unavailable to potential-based methods that only augment rewards.
\end{enumerate}

ESAI trades the theoretical guarantees of potential-based shaping (policy invariance) for adaptive capacity and richer internal dynamics, though this tradeoff requires empirical investigation.

\subsection{Summary of Gaps Addressed}

While individual components exist in isolation—counterfactual reasoning for credit assignment \citep{foerster2018counterfactual,zhou2022pac}, graph diffusion for communication \citep{jiang2020graph}, attention for coordination \citep{iqbal2019actor,liu2020g2anet,wang2021qplex,liu2023na2q}, external shields \citep{alshiekh2018safe,chatterji2025smarl}, intrinsic social preferences \citep{hughes2018inequity,alamiyan2023kindmarl}, ethical environment design \citep{mayoral2025designing}, and differentiable memory \citep{graves2014neural,miconi2018differentiable}—\textbf{no existing work unifies differentiable internal alignment embeddings with counterfactual harm prediction, graph-based propagation, Hebbian memory coupling, and bias-mitigation controls in a single end-to-end trainable architecture}.

ESAI provides this unification as a theoretical framework, formalizing internal alignment embeddings as persistent, predictive, and gradient-coupled states. While related work on kindness measurement \citep{alamiyan2023kindmarl} shares similar motivations, ESAI contributes a complete differentiable architecture for learning alignment through counterfactual forecasting. Empirical validation is needed to assess whether ESAI's integration provides practical advantages over simpler alternatives such as external shielding, heuristic social preferences, or environmental design approaches.

\section{Embedded Safety-Aligned Intelligence: Framework Definition}
\label{sec:framework}

\begin{definition}[Internal Alignment Embedding]
\label{def:iae}
An \emph{internal alignment embedding} (IAE) is a differentiable latent variable $E_t \in \mathbb{R}^k$ maintained by each agent that satisfies:
\begin{enumerate}[itemsep=2pt]
  \item \textbf{Predictive correspondence}: $E_t$ correlates with predicted externalized harm under learned dynamics.
  \item \textbf{Gradient coupling}: $E_t$ influences policy gradients via differentiable transformations of rewards or observations.
  \item \textbf{Temporal persistence}: $E_t$ evolves through learned update rules that preserve information across multiple timesteps.
\end{enumerate}
\end{definition}

\begin{definition}[Embedded Safety-Aligned Intelligence]
\label{def:esai}
A multi-agent learning system exhibits \emph{Embedded Safety-Aligned Intelligence} if:
\begin{enumerate}[itemsep=2pt]
  \item Each agent maintains an IAE satisfying Definition~\ref{def:iae}.
  \item Policy learning incorporates a differentiable alignment objective that penalizes predicted harm encoded in IAE.
  \item IAE dynamics are supervised (implicitly or explicitly) to forecast alignment-relevant outcomes.
  \item The system supports distributed coordination through IAE propagation across agent neighborhoods.
\end{enumerate}
\end{definition}

ESAI systems differ from external alignment mechanisms (reward shaping, constraints, shields) in that alignment pressure arises from \emph{internal learned representations} rather than external supervisory signals. This embedding has three potential advantages:

\begin{itemize}[leftmargin=*]
  \item \textbf{Gradient-based adaptation}: Differentiability enables online learning of alignment objectives without discrete switching or constraint satisfaction solvers.
  \item \textbf{Perceptual salience}: IAE can gate attention to alignment-relevant features, potentially improving sample efficiency in sparse-harm environments.
  \item \textbf{Distributed coordination}: Graph diffusion of IAE enables decentralized alignment pressure without centralized oversight.
\end{itemize}

However, ESAI also introduces challenges: (1) IAE semantics depend on quality of harm supervision, (2) learned embeddings may lack interpretability, and (3) computational overhead increases with embedding dimension and graph connectivity. We formalize these tradeoffs in subsequent sections.

\section{Architecture Overview}
\label{sec:arch}

We present one possible instantiation of the ESAI framework. Alternative architectures satisfying Definitions~\ref{def:iae}--\ref{def:esai} may exist; this section illustrates design principles rather than prescribing a unique implementation.

Each agent augments a standard policy-gradient learner with an internal alignment embedding $E_t \in \mathbb{R}^k$. All computations involving $E_t$ are differentiable, enabling gradient propagation into policy parameters $\theta$.

\subsection{Design Principles}

Table~\ref{tab:modules} maps architectural components to known MARL failure modes, motivating the integration of four mechanisms:

\begin{table}[h]
\centering
\caption{Architectural modules and targeted coordination failures.}
\label{tab:modules}
\begin{tabular}{lll}
\toprule
Module & Targeted Failure Mode & Mechanism \\
\midrule
Counterfactual Regret & Selfish Nash equilibria & Harm prediction penalty \\
IAE-Weighted Attention & Ignoring victim states & Salience modulation \\
Hebbian Memory & Temporal credit assignment & Associative traces \\
Graph Diffusion & Coordination collapse & Multi-agent propagation \\
\bottomrule
\end{tabular}
\end{table}

\begin{figure}[t]
\centering
\includegraphics[width=\linewidth]{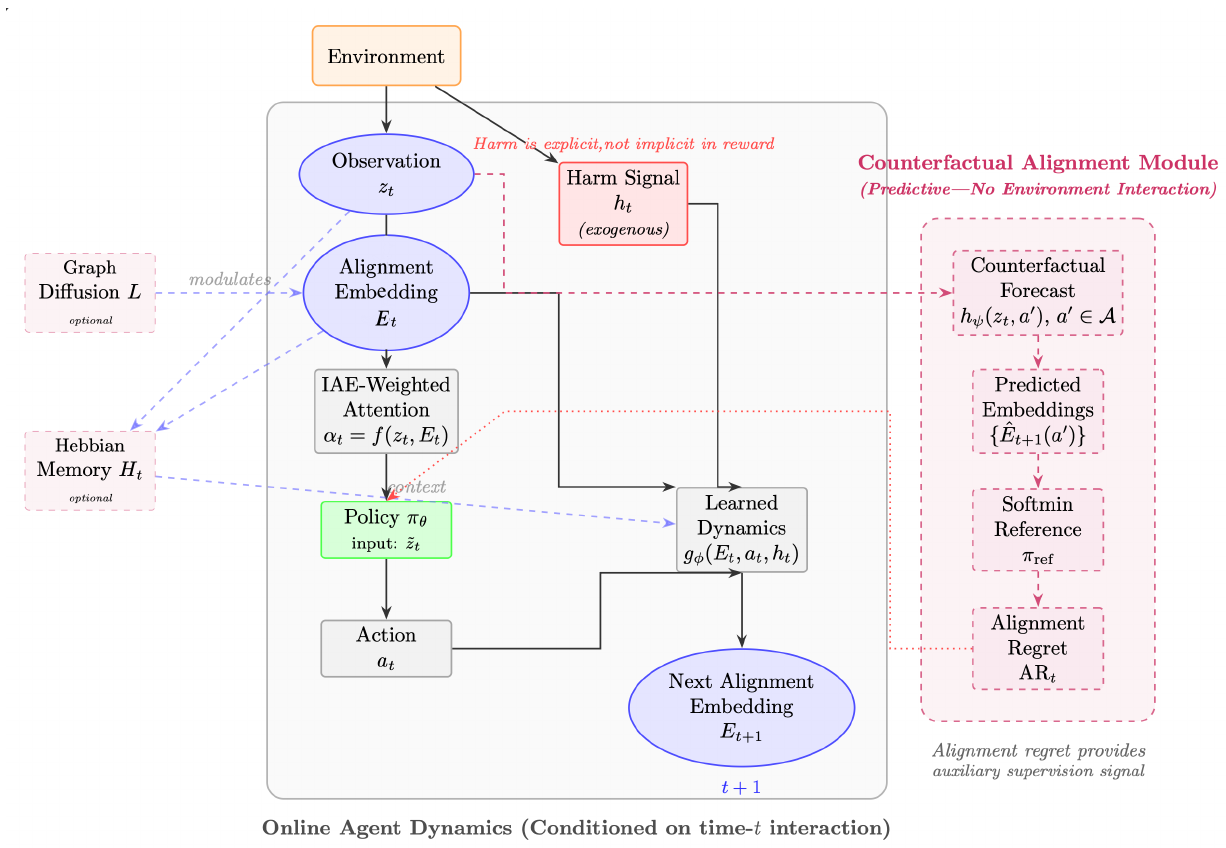}  
\caption{ESAI architecture schematic. The internal alignment embedding $E_t$ is updated via learned dynamics $g_\phi$, graph diffusion $L$, and Hebbian memory $H_t$. Counterfactual forecasts (computed via EMA target network $\psi_{\text{target}}$) generate alignment penalties $\mathrm{AR}_t$ that shape policy gradients. IAE-weighted attention $\alpha_t$ modulates perceptual input $z_t$. Dotted lines indicate gradient flow; solid lines denote forward computation.}
\label{fig:arch}
\end{figure}

\section{Mathematical Formulation}
\label{sec:math}

\subsection{Task and Alignment Objectives}

Let $s_t \in \mathcal{S}$ denote environment state, $a_t \in \mathcal{A}$ agent action, and $r_t^{\text{ext}}$ extrinsic reward. The standard RL objective is:
\begin{equation}
J_{\text{task}}(\theta) = \mathbb{E}_{\pi_\theta}\left[\sum_{t=0}^\infty \gamma^t r_t^{\text{ext}}\right].
\end{equation}

\begin{assumption}[Harm Observability]
\label{ass:harm}
There exists a measurable harm signal $h_t: \mathcal{S} \times \mathcal{A}^N \times \mathcal{S} \to \mathbb{R}_{\ge 0}$ satisfying:
\begin{enumerate}[itemsep=2pt]
    \item \textbf{Exogeneity}: $h_t$ is defined independently of policy parameters $\theta$
    \item \textbf{Observability}: $h_t$ is computable from transition $(s_t, \mathbf{a}_t, s_{t+1})$
    \item \textbf{Boundedness}: $\sup_{t} h_t \le H_{\max} < \infty$
    \item \textbf{Relevance}: $\mathbb{E}[h_t \mid E_{i,t}]$ is monotonically related to $\|E_{i,t}\|_2$
\end{enumerate}
This assumption encodes domain-specific harm definitions as exogenous inputs. ESAI does not learn \emph{what} constitutes harm—only \emph{how to predict and avoid} externally specified harm. The normative content of $h_t$ must be provided by system designers and is subject to the biases discussed in Sec.~\ref{sec:limitations}.
\end{assumption}

Each agent $i \in \{1, \ldots, N\}$ maintains IAE $E_{i,t} \in \mathbb{R}^k$. We define an alignment potential:
\begin{equation}
\Phi(s_t, E_{i,t}) = v^\top \sigma(W[s_t; E_{i,t}]),
\end{equation}
where $[s_t; E_{i,t}]$ denotes concatenation and $\sigma$ is a nonlinear activation. The alignment objective encourages low-norm embeddings:
\begin{equation}
J_{\text{align}}(\theta) = -\mathbb{E}_{\pi_\theta}\left[\sum_t \gamma^t \sum_{i=1}^N \|E_{i,t}\|_2\right].
\end{equation}

This objective encodes the principle that lower IAE magnitude corresponds to lower predicted harm—a design choice that must be validated empirically.

\subsection{IAE Dynamics}

The embedding for agent $i$ evolves via:
\begin{equation}\label{eq:iae_dynamics}
E_{i,t+1} = \gamma_E E_{i,t} + g_\phi(z_{i,t}, a_{i,t}, r_{i,t}^{\text{ext}}) - \alpha \sum_{j \in \mathcal{N}(i)} L_{ij} E_{j,t},
\end{equation}
where:
\begin{itemize}[leftmargin=*,itemsep=2pt]
  \item $\gamma_E \in [0,1)$ controls temporal persistence
  \item $g_\phi: \mathcal{Z} \times \mathcal{A} \times \mathbb{R} \to \mathbb{R}^k$ is a learned update function (e.g., MLP)
  \item $z_{i,t} \in \mathcal{Z}$ is agent $i$'s perceptual input
  \item $\mathcal{N}(i)$ is agent $i$'s neighborhood in the communication graph
  \item $L_{ij}$ are entries of the normalized graph Laplacian
  \item $\alpha \ge 0$ controls diffusion strength
\end{itemize}

The diffusion term $-\alpha \sum_{j \in \mathcal{N}(i)} L_{ij} E_{j,t}$ propagates alignment-relevant information across agent neighborhoods. This design enables decentralized coordination: agents in similar states with high graph connectivity will have correlated IAE dynamics.

\subsection{Differentiable Counterfactual Alignment Penalty}
\label{sec:counterfactual}

For each candidate action $a \in \mathcal{A}$, agent $i$ forecasts the next-step IAE:
\begin{equation}\label{eq:forecast}
\widehat{E}_{i,t+1}^{(a)} = h_\psi(z_{i,t}, a, r_{i,t}^{\text{ext}}, \text{read}(H_{i,t})),
\end{equation}
where $h_\psi$ is a learned forecast network and $\text{read}(H_{i,t})$ provides Hebbian memory context (defined in Sec.~\ref{sec:hebbian}).

\paragraph{Forecast supervision.}
The forecast network $h_\psi$ is trained to predict actual next-step IAE via the loss:
\begin{equation}\label{eq:forecast_loss}
L_{\text{forecast}} = \mathbb{E}\left[\|h_\psi(z_{i,t}, a_{i,t}, r_{i,t}^{\text{ext}}, \text{read}(H_{i,t})) - E_{i,t+1}\|_2^2\right],
\end{equation}
where $E_{i,t+1}$ is computed from Eq.~\eqref{eq:iae_dynamics} using the executed action $a_{i,t}$.

We scalarize forecasted harm via:
\begin{equation}
R(a) = \|\widehat{E}_{i,t+1}^{(a)}\|_2.
\end{equation}

To avoid non-differentiable $\arg\min$, we define a softmin reference distribution with temperature $\tau$:
\begin{equation}\label{eq:softmin}
\pi_{\text{ref}}(a \mid s_t) = \frac{\exp(-R(a)/\tau)}{\sum_{a'} \exp(-R(a')/\tau)}.
\end{equation}

The expected reference embedding is:
\begin{equation}\label{eq:ref_emb}
\widehat{E}_{i,t+1}^{\text{ref}} = \sum_a \pi_{\text{ref}}(a \mid s_t) \, \widehat{E}_{i,t+1}^{(a)}.
\end{equation}

The differentiable alignment regret penalizes deviations from the harm-minimizing reference:
\begin{equation}\label{eq:align_regret}
\mathrm{AR}_{i,t} = \big\|E_{i,t+1} - \widehat{E}_{i,t+1}^{\text{ref}}\big\|_2^2 + \kappa \cdot \frac{1}{|\mathcal{N}(i)|} \sum_{j \in \mathcal{N}(i)} \|E_{j,t}\|_2,
\end{equation}
where the second term uses lagged neighbor embeddings $E_{j,t}$ (not future $E_{j,t+1}$) to maintain causal consistency.

\paragraph{Temperature annealing.}
We propose annealing $\tau$ from high to low values over training for three theoretical reasons:
\begin{enumerate}[leftmargin=*,itemsep=0pt]
  \item \textbf{Early exploration}: High $\tau$ smooths $\pi_{\text{ref}}$, preventing premature convergence when forecasts are inaccurate.
  \item \textbf{Gradient stability}: Soft targets reduce gradient variance compared to hard $\arg\min$.
  \item \textbf{Late discretization}: As $\tau \to 0$, $\pi_{\text{ref}}$ recovers near-deterministic behavior, sharpening alignment signals.
\end{enumerate}

\paragraph{Predictor stability via EMA target network.}
To prevent predictor-policy collusion, we propose maintaining an exponential moving average (EMA) target network:
\begin{equation}\label{eq:ema}
\psi_{\text{target}} \leftarrow \tau_{\text{ema}} \psi_{\text{target}} + (1 - \tau_{\text{ema}}) \psi,
\end{equation}
with $\tau_{\text{ema}} \approx 0.995$. All counterfactual forecasts $\widehat{E}_{i,t+1}^{(a)}$ are computed using $h_{\psi_{\text{target}}}$ to stabilize gradient flow, analogous to stability mechanisms in self-supervised learning and double Q-learning.

\subsection{Reward Transformation and Policy Objective}

We transform extrinsic rewards with the alignment penalty:
\begin{equation}\label{eq:reward_transform}
r'_{i,t} = r_{i,t}^{\text{ext}} - \lambda_{\text{reg}} \, \mathrm{AR}_{i,t}.
\end{equation}

Advantages $A_{i,t}$ are computed using generalized advantage estimation (GAE) on $r'_{i,t}$. The policy loss follows PPO-Clip:
\begin{equation}\label{eq:ppo}
L_\pi = \mathbb{E}_t\!\left[\min\left(\rho_t A_{i,t}, \, \text{clip}(\rho_t, 1-\epsilon, 1+\epsilon) A_{i,t}\right)\right],
\end{equation}
where $\rho_t = \pi_\theta(a_{i,t} \mid s_t) / \pi_{\theta_{\text{old}}}(a_{i,t} \mid s_t)$.

\paragraph{Implementation note.}
The PPO-Clip formulation above is one possible instantiation; ESAI principles are compatible with alternative policy gradient methods (A2C, TRPO, SAC, MPO). Appendix~\ref{app:algorithm} provides a complete training loop using PPO as an illustrative—not prescriptive—example.

\subsection{IAE-Weighted Attention}
\label{sec:attention}

Attention weights are computed from the current IAE to bias perceptual salience. To ensure dimensional consistency, we use projection matrix $W_a \in \mathbb{R}^{d \times k}$:
\begin{equation}\label{eq:attention}
\alpha_{i,t} = \text{softmax}(W_a E_{i,t} + b_a) \in \mathbb{R}^d, \quad \tilde{z}_{i,t} = \alpha_{i,t} \odot z_{i,t} \in \mathbb{R}^d,
\end{equation}
where $\odot$ denotes element-wise product and $z_{i,t} \in \mathbb{R}^d$ is the observation vector.

\textbf{Theoretical motivation}: In sparse-harm environments, agents must attend to low-probability events (e.g., victim states). IAE-weighted attention provides a differentiable mechanism for learned salience modulation that could improve sample efficiency compared to uniform attention—though this remains an empirical question.

\subsection{Hebbian Affect-Memory Coupling}
\label{sec:hebbian}

A Hebbian memory matrix $H_{i,t} \in \mathbb{R}^{k \times d}$ updates via outer-product learning:
\begin{equation}\label{eq:hebbian}
H_{i,t+1} = (1 - \delta_H) H_{i,t} + \eta_H (E_{i,t} \otimes z_{i,t}),
\end{equation}
where $\delta_H > 0$ ensures decay and $\eta_H$ is the learning rate.

The Hebbian trace supports counterfactual forecasting via differentiable read:
\begin{equation}\label{eq:hebbian_read}
\text{read}(H_{i,t}) = W_r \, \text{vec}(H_{i,t}),
\end{equation}
where $\text{vec}(\cdot)$ flattens the matrix. This read vector enters forecast model $h_\psi$ in Eq.~\eqref{eq:forecast}.

\textbf{Theoretical motivation}: Hebbian traces encode historical co-activations of IAE and perceptual features, providing context for counterfactual prediction. This could improve temporal credit assignment by associating past states with delayed harm outcomes.

\subsection{Graph Diffusion and Similarity Weighting}
\label{sec:diffusion}

The graph Laplacian $L$ in Eq.~\eqref{eq:iae_dynamics} is row-normalized with spectral radius $\rho(L) \le 2$. Diffusion weights are modulated by cosine similarity of learned identity embeddings $\phi_i \in \mathbb{R}^{d_{\text{id}}}$:
\begin{equation}\label{eq:similarity}
\beta_{ij} = \max(0, \cos(\phi_i, \phi_j)).
\end{equation}

To mitigate emergent in-group favoritism, we introduce a similarity-suppression regularizer:
\begin{equation}\label{eq:bias_reg}
L_{\text{bias}} = \lambda_{\text{bias}} \| \tilde{A} \odot S \|_F^2,
\end{equation}
where $S$ is the similarity matrix with entries $S_{ij} = \beta_{ij}$ and $\tilde{A}$ is the weighted adjacency.

\textbf{Theoretical motivation}: Similarity-weighted diffusion enables faster propagation among similar agents, potentially accelerating coordination. However, this also creates risk of in-group bias. The regularizer $L_{\text{bias}}$ provides a tunable fairness-performance tradeoff.

\subsection{Complete Optimization Objective}

The full objective combines policy loss, entropy regularization, and alignment penalties:
\begin{equation}\label{eq:full_objective}
L(\theta) = \mathbb{E}_t\!\left[
  L_\pi - \beta H(\pi_\theta) + \lambda_H \|H_{i,t}\|_2^2 + \lambda_D \|L\|_F^2 + \lambda_{\text{bias}} L_{\text{bias}}
\right],
\end{equation}
where $\beta$ is entropy coefficient, $\lambda_H$ regularizes Hebbian trace norms, $\lambda_D$ controls Laplacian smoothness, and $\lambda_{\text{bias}}$ governs fairness.

\section{Theoretical Properties and Stability}
\label{sec:stability}

\begin{proposition}[Bounded IAE Under Contraction Condition]
\label{prop:contraction}
Consider the IAE update in Eq.~\eqref{eq:iae_dynamics}. Assume:
\begin{enumerate}[itemsep=2pt]
  \item $g_\phi$ is $L_g$-Lipschitz continuous: $\|g_\phi(z, a, r) - g_\phi(z', a', r')\|_2 \le L_g(\|z-z'\|_2 + \|a-a'\|_2 + |r-r'|)$
  \item Inputs are bounded: $\|z_{i,t}\|_2 \le C_z$, $\|a_{i,t}\|_2 \le C_a$, $|r_{i,t}| \le C_r$ for all $i, t$
  \item There exists $B_0$ such that $\|g_\phi(0, 0, 0)\|_2 \le B_0$
  \item The spectral condition holds: $\|\gamma_E I - \alpha L\|_2 + L_g < 1$
\end{enumerate}
Then $\sup_t \|E_{i,t}\|_2 < \infty$ for all agents $i$ and trajectories.
\end{proposition}

\begin{proof}[Proof sketch]
By Lipschitz continuity and boundedness:
\begin{align*}
\|g_\phi(z_{i,t}, a_{i,t}, r_{i,t})\|_2 &\le \|g_\phi(z_{i,t}, a_{i,t}, r_{i,t}) - g_\phi(0, 0, 0)\|_2 + \|g_\phi(0, 0, 0)\|_2 \\
&\le L_g(C_z + C_a + C_r) + B_0 =: K.
\end{align*}

Taking norms in Eq.~\eqref{eq:iae_dynamics}:
\begin{align*}
\|E_{i,t+1}\|_2 &\le \left\|\gamma_E E_{i,t} - \alpha \sum_{j \in \mathcal{N}(i)} L_{ij} E_{j,t}\right\|_2 + K \\
  &\le \|\gamma_E I - \alpha L\|_2 \max_j \|E_{j,t}\|_2 + K.
\end{align*}

Let $\rho = \|\gamma_E I - \alpha L\|_2 < 1 - L_g$ by assumption. Then iterating:
\[
\max_i \|E_{i,t}\|_2 \le \rho^t \max_i \|E_{i,0}\|_2 + \frac{K}{1 - \rho} < \infty.
\]
The spectral condition can be enforced via spectral normalization of $L$ and gradient clipping on $g_\phi$.
\end{proof}

\begin{remark}[Boundedness Does Not Imply Alignment]
Proposition~\ref{prop:contraction} guarantees that IAE remains bounded, preventing numerical instability. However, bounded $E_{i,t}$ is a \emph{necessary but not sufficient} condition for aligned behavior. Convergence to socially optimal equilibria—where agents jointly minimize harm while maximizing task performance—remains an open theoretical problem. The stability result ensures the learning process is well-defined; it does not guarantee the learned policy is aligned.
\end{remark}

\begin{proposition}[Hebbian Trace Stability]
\label{prop:hebbian}
If the state-IAE joint distribution admits bounded second moments $\mathbb{E}[\|E_{i,t}\|_2^2], \mathbb{E}[\|z_{i,t}\|_2^2] < \infty$ and the design constraint
\begin{equation}
\delta_H > \eta_H \sqrt{\mathbb{E}[\|E_{i,t}\|_2^2] \mathbb{E}[\|z_{i,t}\|_2^2]}
\end{equation}
holds, then the Hebbian trace converges in mean-square sense to a bounded fixed point $\mathbb{E}[\|H_{i,t}\|_F] < C$ for some constant $C$.
\end{proposition}

\begin{proof}[Proof sketch]
Taking expectations and norms in Eq.~\eqref{eq:hebbian}:
\begin{align*}
\mathbb{E}[\|H_{i,t+1}\|_F] &\le (1 - \delta_H) \mathbb{E}[\|H_{i,t}\|_F] + \eta_H \mathbb{E}[\|E_{i,t}\|_2 \|z_{i,t}\|_2].
\end{align*}
If $\delta_H > \eta_H \mathbb{E}[\|E_{i,t}\|_2 \|z_{i,t}\|_2] \le \eta_H \sqrt{\mathbb{E}[\|E_{i,t}\|_2^2] \mathbb{E}[\|z_{i,t}\|_2^2]}$ by Cauchy-Schwarz, this forms a contraction with fixed point:
\[
\mathbb{E}[\|H^*\|_F] = \frac{\eta_H \mathbb{E}[\|E_{i,t}\|_2 \|z_{i,t}\|_2]}{\delta_H}.
\]
Full proof with mean-square convergence is in Appendix~\ref{app:hebbian_proof}.
\end{proof}

\subsection{Computational Complexity}
\label{sec:complexity}

The ESAI forward pass involves per agent:
\begin{enumerate}[leftmargin=*,itemsep=2pt]
  \item IAE dynamics update: $O(k^2 + kd + |\mathcal{N}(i)|k)$
  \item Counterfactual forecasts: $O(|\mathcal{A}| \cdot (k^2 + kd))$
  \item Attention gating: $O(kd)$
  \item Hebbian update: $O(kd)$
  \item Graph diffusion: $O(|\mathcal{N}(i)| k)$
\end{enumerate}

For bounded-degree graphs with $|\mathcal{N}(i)| = O(1)$ and $N$ agents, total complexity is $O(N|\mathcal{A}|k^2 + N|\mathcal{A}|kd)$ per step. Counterfactual forecasting dominates; this can be mitigated via top-$K$ action sampling, reducing $|\mathcal{A}|$ to $K \ll |\mathcal{A}|$.

\section{Discussion}
\label{sec:discussion}

\subsection{Potential Advantages of Embedded Alignment}

ESAI's theoretical design suggests several potential advantages over external alignment mechanisms:

\paragraph{Gradient-based harm forecasting.}
Unlike discrete safety constraints or non-differentiable shields, IAE enables continuous gradient flow from predicted harm to policy parameters. This could enable online adaptation to novel harm types without manual constraint redesign—though empirical validation is needed.

\paragraph{Perceptual salience modulation.}
IAE-weighted attention (Eq.~\ref{eq:attention}) provides a mechanism for learned salience toward alignment-relevant features. In sparse-harm environments where victim states are rare, this could improve sample efficiency compared to uniform attention—but may also introduce attention collapse risks.

\paragraph{Decentralized coordination.}
Graph diffusion of IAE enables distributed alignment pressure: agents in connected neighborhoods will develop correlated harm predictions without centralized oversight. This could improve scalability compared to centralized critics, but introduces new failure modes under adversarial graph topology.

\paragraph{Interpretable bias controls.}
The similarity-suppression regularizer $L_{\text{bias}}$ provides an explicit tuning parameter for fairness-performance tradeoffs. This transparency could aid auditing compared to black-box reward shaping.

\subsection{Open Theoretical Questions}

Several fundamental questions remain unresolved:

\paragraph{Optimal embedding dimensionality.}
What is the minimal $k$ required to represent alignment-relevant structure for a given task? Information-theoretic bounds could formalize this, but may depend on environment-specific harm complexity.

\paragraph{Convergence guarantees.}
Proposition~\ref{prop:contraction} ensures bounded IAE but does not guarantee convergence to desirable equilibria. Under what conditions do ESAI policies converge to socially optimal Nash equilibria rather than selfish or collusive ones?

\paragraph{Robustness to distribution shift.}
How do IAE dynamics generalize when deployed in environments with different harm structures than training? Transfer learning theory may provide bounds, but ESAI's coupled dynamics (attention, diffusion, Hebbian) complicate analysis.

\paragraph{Interpretability limits.}
While IAE magnitude is designed to correlate with harm, individual latent dimensions may lack semantic meaning. Can we derive conditions under which IAE factors into interpretable subspaces (e.g., via disentanglement objectives)?

\subsection{Comparison to Alternative Paradigms}

\paragraph{Versus reward shaping.}
Potential-based shaping \citep{ng1999policy} provides theoretical guarantees (policy invariance under certain conditions) that ESAI lacks. However, ESAI's learned dynamics could adapt to non-stationary harm structures where hand-designed potentials fail. The tradeoff between theoretical guarantees and adaptive capacity requires empirical investigation.

\paragraph{Versus constrained optimization.}
CPO \citep{achiam2017constrained} enforces hard safety constraints via trust regions. ESAI trades hard guarantees for soft, differentiable penalties that enable gradient-based learning. In high-dimensional action spaces where constraint satisfaction is expensive, ESAI's continuous relaxation may offer computational advantages—but at the cost of occasional constraint violations.

\paragraph{Versus multi-objective RL.}
Multi-objective methods optimize Pareto frontiers over task and safety objectives. ESAI implicitly performs multi-objective optimization via $\lambda_{\text{reg}}$ (Eq.~\ref{eq:reward_transform}) but adds internal dynamics (attention, diffusion) unavailable to standard scalarization approaches. Whether this added complexity provides practical benefits is an empirical question.

\section{Limitations and Assumptions}
\label{sec:limitations}

\paragraph{Dependence on harm specification.}
ESAI assumes availability of domain-specific harm metrics satisfying Assumption~\ref{ass:harm}. The framework learns to \emph{predict and avoid} harm but does not learn \emph{what constitutes harm}—this normative content must be externally specified. In practice, defining harm is a normative choice that encodes cultural values and may not generalize across contexts. Deployment requires explicit ethical review of harm operationalization and potential biases.

\paragraph{Demonstration quality.}
The counterfactual reference policy $\pi_{\text{ref}}$ (Eq.~\ref{eq:softmin}) depends on quality of harm forecasts, which may be trained on demonstrations. Biased or low-quality demonstrations will produce misaligned IAE attractors. Robustness to demonstration noise and cultural variance in harm definitions are critical open problems.

\paragraph{Computational overhead.}
Counterfactual forecasting requires $|\mathcal{A}|$ forward passes per decision, increasing wall-clock time. For large action spaces, approximate methods (top-$K$ sampling, learned action proposals) are needed but may sacrifice alignment quality.

\paragraph{Architectural complexity.}
ESAI integrates four mechanisms (counterfactual regret, attention, Hebbian memory, graph diffusion), each introducing hyperparameters ($\tau$, $\gamma_E$, $\delta_H$, $\alpha$, $\lambda_{\text{bias}}$). Whether simpler architectures achieve comparable alignment with fewer degrees of freedom is unknown. The principle of parsimony suggests validating each component's necessity empirically.

\paragraph{Adversarial robustness.}
Similarity-weighted diffusion creates vulnerability to adversarial graph topology or identity embedding manipulation. An adversary could induce worst-case in-group bias by manipulating $\phi_i$ to maximize similarity with preferred agents. Defensive mechanisms (anomaly detection, periodic similarity audits) are needed for deployment.

\paragraph{Lack of empirical validation.}
\textbf{This paper presents a theoretical framework without experimental validation.} Claims about advantages (sample efficiency, scalability, bias mitigation) are conjectural pending empirical investigation across diverse environments, baselines, and metrics. Key empirical questions include:
\begin{itemize}[leftmargin=*,itemsep=2pt]
  \item Does ESAI improve cooperation metrics compared to reward shaping, CPO, or multi-objective RL?
  \item Do IAE dynamics exhibit interpretable attractor structure correlated with ground-truth harm?
  \item Does IAE-weighted attention improve sample efficiency in sparse-harm scenarios?
  \item How does performance scale with embedding dimension $k$, agent count $N$, and graph connectivity?
  \item Can interventional probes (manually perturbing $E_{i,t}$) demonstrate causal regulatory effects?
\end{itemize}

\paragraph{Evaluation scope.}
Even with future empirical work, evaluation will likely focus on relatively low-dimensional environments (gridworlds, simple continuous control). Extension to high-dimensional robotic control, realistic physics, or human-AI interaction remains speculative. We make no claims about real-world deployment readiness.

\paragraph{Theoretical gaps.}
Propositions~\ref{prop:contraction}--\ref{prop:hebbian} provide sufficient conditions for bounded IAE and Hebbian traces but do not address:
\begin{itemize}[leftmargin=*,itemsep=2pt]
  \item Convergence to socially optimal equilibria
  \item Sample complexity bounds
  \item Regret bounds relative to optimal alignment policies
  \item Robustness guarantees under distribution shift
\end{itemize}
Formal analysis of these properties is important future work.

\section{Broader Impact and Ethical Considerations}

\paragraph{Potential benefits.}
If empirically validated, ESAI could contribute to:
\begin{itemize}[leftmargin=*,itemsep=2pt]
  \item Safer exploration in cooperative robotics by providing intrinsic harm-avoidance pressure
  \item Reduced need for extensive human preference labeling via counterfactual self-supervision
  \item Interpretable bias-detection controls via similarity auditing
  \item Improved coordination stability under distribution shift
\end{itemize}

\paragraph{Potential risks.}
Deployment of ESAI-based systems could create risks:
\begin{itemize}[leftmargin=*,itemsep=2pt]
  \item \textbf{Bias amplification}: Similarity-weighted diffusion can encode and amplify in-group favoritism if not carefully regularized
  \item \textbf{Normative lock-in}: Demonstration-driven counterfactual supervision entrenches the values of demonstration generators, risking marginalization of minority perspectives
  \item \textbf{Dual-use potential}: Differentiable coordination mechanisms could be repurposed for adversarial applications (collusion, deception)
  \item \textbf{Opacity}: Despite attention mechanisms, IAE-driven decisions may remain opaque to end-users in high-stakes contexts
\end{itemize}

\paragraph{Governance requirements.}
Before any real-world deployment, ESAI systems require:
\begin{itemize}[leftmargin=*,itemsep=2pt]
  \item Domain-specific safety audits validating harm metric definitions
  \item Diverse demonstration sources representing multiple cultural and normative perspectives
  \item Continuous monitoring of similarity-conditioned cooperation rates for bias detection
  \item Adversarial stress testing under worst-case graph topology and identity perturbations
  \item Human-in-the-loop oversight for high-stakes decisions
  \item Compliance with relevant ethical guidelines and regulations
\end{itemize}

\paragraph{Disclaimer.}
IAE is a computational abstraction—\emph{not} a model of subjective consciousness, phenomenal experience, or human emotion. All theoretical analysis uses simulated environments; no human subjects or personal data are involved.

\section{Conclusion}

We introduced \emph{Embedded Safety-Aligned Intelligence} (ESAI), a theoretical framework for multi-agent reinforcement learning that embeds alignment constraints in learned internal representations. Through differentiable internal alignment embeddings (IAE) supervised via counterfactual harm prediction, ESAI provides a conceptual alternative to external reward shaping and non-differentiable safety constraints.

The framework integrates four mechanisms: (1) differentiable counterfactual alignment penalties enabling gradient-based harm forecasting, (2) IAE-weighted attention biasing perceptual salience, (3) Hebbian affect-memory coupling supporting temporal credit assignment, and (4) similarity-weighted graph diffusion with bias-mitigation controls. We derived conditions for bounded internal embeddings under Lipschitz constraints and spectral radius bounds, analyzed computational complexity as $O(N|\mathcal{A}|kd)$, and discussed theoretical properties including contraction dynamics and fairness-performance tradeoffs.

\paragraph{Key contributions.}
\begin{itemize}[leftmargin=*]
  \item Formal definition of internal alignment embeddings and embedded safety-aligned intelligence (Definitions~\ref{def:iae}--\ref{def:esai})
  \item Differentiable counterfactual alignment penalty via softmin reference distributions
  \item Integration of IAE with attention gating, Hebbian memory, and graph diffusion
  \item Stability analysis and complexity bounds (Propositions~\ref{prop:contraction}--\ref{prop:hebbian})
  \item Identification of open theoretical questions and empirical validation requirements
\end{itemize}

\paragraph{Limitations.}
ESAI is a conceptual framework without empirical validation. Claims about advantages (sample efficiency, scalability, bias mitigation) are conjectural. Theoretical analysis provides bounded IAE but not convergence guarantees to socially optimal equilibria. The framework assumes availability of harm metrics and demonstrations, introducing normative dependencies and bias risks.

\paragraph{Future work.}
Critical next steps include:
\begin{itemize}[leftmargin=*]
  \item \textbf{Empirical validation}: Test ESAI instantiations across diverse MARL benchmarks (cooperative navigation, social dilemmas, human-AI coordination tasks) with statistical comparisons to reward shaping, CPO, and multi-objective RL
  \item \textbf{Theoretical analysis}: Derive convergence guarantees, sample complexity bounds, and regret bounds relative to optimal alignment policies
  \item \textbf{Interpretability}: Develop methods to ground IAE semantics (e.g., via correlation with domain harm metrics, interventional probes, disentanglement objectives)
  \item \textbf{Robustness}: Characterize performance under distribution shift, adversarial perturbations, and graph topology attacks
  \item \textbf{Scalability}: Evaluate on high-dimensional continuous control and large-scale multi-agent scenarios ($N > 50$)
  \item \textbf{Human studies}: Assess subjective preference for ESAI-trained agents in human-AI coordination tasks
\end{itemize}

ESAI demonstrates that embedding alignment in internal learned representations is a theoretically coherent paradigm. Whether this translates to practical advantages over simpler external mechanisms remains an empirical question requiring rigorous evaluation.

\section*{Acknowledgments}

We thank the open-source communities (PyTorch, OpenAI Gym, PettingZoo) whose tooling will enable future empirical validation. This theoretical work was conducted independently as part of academic coursework with no external funding.

\section*{Reproducibility Statement}

This paper presents a theoretical framework without experimental implementation. Upon future empirical validation, we commit to releasing all code, model checkpoints, hyperparameter configurations, and evaluation protocols under an open-source license (MIT). Theoretical derivations and mathematical formulations are fully specified in Sections~\ref{sec:math}--\ref{sec:stability}, with complete proofs in Appendices~\ref{app:derivations}--\ref{app:hebbian_proof}.

\bibliographystyle{plainnat}
\bibliography{references}

\clearpage
\appendix

\section{Appendix A: Detailed Mathematical Derivations}
\label{app:derivations}

\subsection{Gradient Flow Through Softmin Reference}

Consider the gradient of alignment regret (Eq.~\ref{eq:align_regret}) with respect to forecast parameters $\psi$:
\begin{align}
\frac{\partial \mathrm{AR}_{i,t}}{\partial \psi} &= \frac{\partial}{\partial \psi} \big\|E_{i,t+1} - \widehat{E}_{i,t+1}^{\text{ref}}\big\|_2^2 \\
&= -2 \big(E_{i,t+1} - \widehat{E}_{i,t+1}^{\text{ref}}\big)^\top \frac{\partial \widehat{E}_{i,t+1}^{\text{ref}}}{\partial \psi}.
\end{align}

From Eq.~\eqref{eq:ref_emb}:
\begin{align}
\frac{\partial \widehat{E}_{i,t+1}^{\text{ref}}}{\partial \psi} &= \sum_a \left[ \frac{\partial \pi_{\text{ref}}(a)}{\partial \psi} \widehat{E}_{i,t+1}^{(a)} + \pi_{\text{ref}}(a) \frac{\partial \widehat{E}_{i,t+1}^{(a)}}{\partial \psi} \right].
\end{align}

The softmin gradient is:
\begin{align}
\frac{\partial \pi_{\text{ref}}(a)}{\partial \psi} &= -\frac{1}{\tau} \pi_{\text{ref}}(a) \left[ \frac{\partial R(a)}{\partial \psi} - \sum_{a'} \pi_{\text{ref}}(a') \frac{\partial R(a')}{\partial \psi} \right],
\end{align}
where $R(a) = \|\widehat{E}_{i,t+1}^{(a)}\|_2$.

For the norm gradient:
\begin{align}
\frac{\partial R(a)}{\partial \psi} &= \frac{\partial}{\partial \psi} \|\widehat{E}_{i,t+1}^{(a)}\|_2 \\
&= \frac{\widehat{E}_{i,t+1}^{(a) \top}}{\|\widehat{E}_{i,t+1}^{(a)}\|_2} \frac{\partial \widehat{E}_{i,t+1}^{(a)}}{\partial \psi}.
\end{align}

Combining these terms:
\begin{align}
\frac{\partial \widehat{E}_{i,t+1}^{\text{ref}}}{\partial \psi} &= \sum_a \pi_{\text{ref}}(a) \frac{\partial \widehat{E}_{i,t+1}^{(a)}}{\partial \psi} \notag \\
&\quad - \frac{1}{\tau} \sum_a \pi_{\text{ref}}(a) \widehat{E}_{i,t+1}^{(a)} \left[ \frac{\widehat{E}_{i,t+1}^{(a) \top}}{\|\widehat{E}_{i,t+1}^{(a)}\|_2} \frac{\partial \widehat{E}_{i,t+1}^{(a)}}{\partial \psi} - \mathbb{E}_{\pi_{\text{ref}}}[\cdot] \right].
\end{align}

This decomposition shows that gradients flow through both the forecast magnitudes $R(a)$ and the predicted embeddings $\widehat{E}_{i,t+1}^{(a)}$, enabling joint learning of harm prediction and embedding dynamics. The temperature $\tau$ controls the sharpness of the gradient signal: low $\tau$ produces sparse gradients concentrated on the minimum-harm action, while high $\tau$ distributes gradients across multiple actions.

\subsection{Spectral Radius Constraint Enforcement}

To satisfy the spectral condition in Proposition~\ref{prop:contraction}, we enforce:
\[
\rho(\gamma_E I - \alpha L) \le \rho_{\max} < 1 - L_g.
\]

For symmetric Laplacian $L$ with eigenvalues $0 = \lambda_1 \le \lambda_2 \le \cdots \le \lambda_N$, the eigenvalues of $\gamma_E I - \alpha L$ are:
\[
\mu_i = \gamma_E - \alpha \lambda_i, \quad i = 1, \ldots, N.
\]

The spectral radius is:
\begin{align}
\rho(\gamma_E I - \alpha L) &= \max_{i} |\mu_i| \\
&= \max\{|\gamma_E - \alpha \lambda_2|, |\gamma_E - \alpha \lambda_N|\}.
\end{align}

\paragraph{Case 1: $\gamma_E \ge \alpha \lambda_N$.}
In this case, $\mu_i \ge 0$ for all $i$, so:
\[
\rho(\gamma_E I - \alpha L) = \gamma_E - \alpha \lambda_2 = \gamma_E,
\]
since $\lambda_2 \approx 0$ for connected graphs. The spectral condition becomes $\gamma_E < 1 - L_g$.

\paragraph{Case 2: $\gamma_E < \alpha \lambda_N$.}
The largest magnitude eigenvalue is:
\[
\rho(\gamma_E I - \alpha L) = \alpha \lambda_N - \gamma_E,
\]
requiring:
\[
\alpha \le \frac{\gamma_E + \rho_{\max}}{\lambda_N}.
\]

\paragraph{Practical implementation.}
We normalize $L$ via spectral normalization to ensure $\lambda_N \le 2$, then choose:
\[
\alpha < \min\left\{\frac{\gamma_E + \rho_{\max}}{2}, \frac{1 - L_g - \gamma_E}{\lambda_N}\right\}.
\]

For typical values $\gamma_E = 0.9$, $L_g = 0.05$, $\rho_{\max} = 0.95$, this gives $\alpha < \min\{0.925, 0.025\} = 0.025$.

\subsection{Lipschitz Constant Enforcement for $g_\phi$}

To enforce $L_g$-Lipschitz continuity of the IAE update function $g_\phi$, we use spectral normalization of all weight matrices in the MLP. For a two-layer network:
\[
g_\phi(z, a, r) = W_2 \sigma(W_1 [z; a; r] + b_1) + b_2,
\]
the Lipschitz constant is bounded by:
\[
L_g \le \|W_2\|_2 \cdot L_\sigma \cdot \|W_1\|_2,
\]
where $L_\sigma$ is the Lipschitz constant of the activation function $\sigma$.

For ReLU activations, $L_\sigma = 1$, so we enforce:
\[
\|W_1\|_2 \le \sqrt{L_g}, \quad \|W_2\|_2 \le \sqrt{L_g}.
\]

This is achieved via spectral normalization: replace $W_i$ with $W_i / \max(\|W_i\|_2, \sqrt{L_g})$ after each gradient update.

\section{Appendix B: Alternative Architectural Instantiations}
\label{app:alternatives}

The ESAI framework (Definitions~\ref{def:iae}--\ref{def:esai}) admits alternative implementations beyond the architecture presented in Sec.~\ref{sec:arch}. We describe three variants:

\subsection{Variational IAE with Uncertainty Quantification}

Replace deterministic $E_{i,t}$ with stochastic $E_{i,t} \sim q_\phi(E \mid z_{i,t}, a_{i,t})$ and optimize a variational lower bound:
\begin{equation}
\mathcal{L}_{\text{VAE}} = \mathbb{E}_{q_\phi} \left[ \log p(E_{i,t+1} \mid E_{i,t}, z_{i,t}, a_{i,t}) \right] - \beta \, \text{KL}(q_\phi(E_{i,t}) \| p_{\text{prior}}(E)),
\end{equation}
where $p_{\text{prior}}$ is a standard Gaussian prior.

\paragraph{Advantages:}
\begin{itemize}[leftmargin=*,itemsep=2pt]
  \item Uncertainty quantification over harm predictions
  \item Robust to noisy observations via posterior averaging
  \item Natural exploration via stochastic embeddings
\end{itemize}

\paragraph{Disadvantages:}
\begin{itemize}[leftmargin=*,itemsep=2pt]
  \item Increased computational cost (reparameterization trick, KL computation)
  \item Additional hyperparameter $\beta$ (KL weight)
  \item May increase variance in policy gradients
\end{itemize}

\subsection{Transformer-Based Counterfactual Forecasting}

Replace MLP forecast network $h_\psi$ with a transformer operating over action sequences:
\begin{equation}
\widehat{E}_{i,t+1}^{(a_1, \ldots, a_H)} = \text{Transformer}\left(\left[z_{i,t}; a_1; a_2; \cdots; a_H\right]\right),
\end{equation}
enabling multi-step lookahead counterfactuals over horizon $H$.

\paragraph{Formulation:}
The transformer takes as input a sequence of length $H+1$:
\begin{itemize}[leftmargin=*,itemsep=2pt]
  \item Token 0: Current observation $z_{i,t}$ (embedded via learned projection)
  \item Tokens 1 to $H$: Candidate action sequence $a_1, \ldots, a_H$
\end{itemize}

Output is the predicted IAE after executing the action sequence:
\[
\widehat{E}_{i,t+H}^{(a_{1:H})} = \text{Linear}(\text{TransformerEncoder}([z_{i,t}; a_{1:H}])[0]),
\]
where $[0]$ denotes the first output token.

\paragraph{Advantages:}
\begin{itemize}[leftmargin=*,itemsep=2pt]
  \item Captures long-horizon harm accumulation
  \item Attention mechanism provides interpretability over action sequences
  \item Can model complex temporal dependencies
\end{itemize}

\paragraph{Disadvantages:}
\begin{itemize}[leftmargin=*,itemsep=2pt]
  \item Exponential complexity $O(|\mathcal{A}|^H)$ for horizon $H > 1$
  \item Requires significant training data for stable learning
  \item Increased memory footprint for attention matrices
\end{itemize}

\subsection{Continuous Action Spaces via Gaussian Mixtures}

For continuous action spaces $a \in \mathbb{R}^m$, replace discrete softmin (Eq.~\ref{eq:softmin}) with continuous distribution:
\begin{equation}
\pi_{\text{ref}}(a) = \frac{1}{Z} \exp(-R(a)/\tau), \quad Z = \int_{\mathbb{R}^m} \exp(-R(a)/\tau) \, da,
\end{equation}
where $R(a) = \|\widehat{E}_{i,t+1}^{(a)}\|_2$ as before.

The expected reference embedding becomes:
\begin{equation}
\widehat{E}_{i,t+1}^{\text{ref}} = \frac{1}{Z} \int_{\mathbb{R}^m} \exp(-R(a)/\tau) \widehat{E}_{i,t+1}^{(a)} \, da.
\end{equation}

\paragraph{Approximation via sampling:}
Since the integral is intractable, approximate via Monte Carlo sampling:
\begin{align}
\widehat{E}_{i,t+1}^{\text{ref}} &\approx \frac{1}{K} \sum_{k=1}^K w_k \widehat{E}_{i,t+1}^{(a_k)}, \\
w_k &= \frac{\exp(-R(a_k)/\tau)}{\sum_{j=1}^K \exp(-R(a_j)/\tau)},
\end{align}
where $a_k \sim q(a)$ are samples from a proposal distribution $q$ (e.g., current policy).

\paragraph{Advantages:}
\begin{itemize}[leftmargin=*,itemsep=2pt]
  \item Extends ESAI to continuous control domains
  \item Maintains differentiability via reparameterization trick
  \item Scales to high-dimensional action spaces
\end{itemize}

\paragraph{Disadvantages:}
\begin{itemize}[leftmargin=*,itemsep=2pt]
  \item Requires careful choice of proposal distribution $q$
  \item Variance in Monte Carlo estimate can destabilize training
  \item May need adaptive temperature schedules for convergence
\end{itemize}

\section{Appendix C: Illustrative Implementation (Non-Normative)}
\label{app:algorithm}

\textbf{Disclaimer.} This appendix provides one possible instantiation of the ESAI framework using PPO-Clip. The specific algorithmic choices (PPO vs. other policy gradient methods, hyperparameter values, update schedules) are illustrative, not prescriptive. Alternative implementations satisfying Definitions~\ref{def:iae}--\ref{def:esai} are equally valid. Empirical validation is required before adopting any specific configuration.

\subsection{Hyperparameter Configuration}

Typical hyperparameter values for the algorithm are shown in Table~\ref{tab:hyperparams}.

\begin{table}[ht]
\centering
\caption{Default hyperparameter configuration for ESAI training.}
\label{tab:hyperparams}
\begin{tabular}{ll}
\toprule
Hyperparameter & Typical Value \\
\midrule
IAE dimension $k$ & 32 \\
Observation dimension $d$ & Environment-dependent (16--128) \\
Persistence $\gamma_E$ & 0.9 \\
Diffusion strength $\alpha$ & 0.05 \\
Alignment penalty $\lambda_{\text{reg}}$ & 0.1 \\
Neighbor weight $\kappa$ & 0.5 \\
Temperature initial $\tau_0$ & 1.0 \\
Temperature minimum $\tau_{\min}$ & 0.01 \\
Annealing timescale $K_\tau$ & $5 \times 10^5$ steps \\
EMA rate $\tau_{\text{ema}}$ & 0.995 \\
Hebbian learning rate $\eta_H$ & $1 \times 10^{-3}$ \\
Hebbian decay $\delta_H$ & 0.02 \\
Bias regularizer $\lambda_{\text{bias}}$ & 0.01 \\
PPO clip $\epsilon$ & 0.2 \\
GAE $\lambda$ & 0.95 \\
Discount $\gamma$ & 0.99 \\
Learning rate & $3 \times 10^{-4}$ \\
Batch size & 2048 \\
\bottomrule
\end{tabular}
\end{table}

\subsection{Complete Training Algorithm}

Algorithm~\ref{alg:esai} presents the complete ESAI training loop for multi-agent systems. Each episode consists of environment interaction, internal state updates, and network optimization. The forecast network is trained on realized transitions and queried counterfactually at inference time. Counterfactual supervision arises from minimizing the discrepancy between realized IAE trajectories and a softmin-weighted reference over forecasted counterfactual IAEs.

\begin{algorithm}[H]
\caption{ESAI Training Loop (Multi-Agent System)}
\label{alg:esai}
\begin{algorithmic}[1]
\REQUIRE Environment $\mathcal{E}$, number of agents $N$, max episodes $M$, horizon $T$
\REQUIRE Hyperparameters: $\gamma_E, \alpha, \lambda_{\text{reg}}, \kappa, \tau_{\text{ema}}, \delta_H, \eta_H, \tau_0, \tau_{\min}, K_\tau$

\STATE \textbf{Initialize:} Policy $\pi_\theta$, value network $V_\omega$, IAE dynamics $g_\phi$
\STATE \textbf{Initialize:} Forecast network $h_\psi$, target $h_{\psi_{\text{target}}} \gets h_\psi$
\STATE \textbf{Initialize:} IAE $E_{i,0} \gets \mathbf{0}$, Hebbian $H_{i,0} \gets \mathbf{0}$ for all $i$
\STATE \textbf{Initialize:} Identity embeddings $\{\phi_i\}_{i=1}^N$, graph Laplacian $L$

\STATE \textcolor{gray}{\textit{// All expectations and norms assumed finite; stability conditions in Sec.~\ref{sec:stability}}}

\FOR{episode $m = 1$ to $M$}
    \STATE Reset environment: $s_0 \sim \rho_0$
    \STATE Reset internal states: $E_{i,0} \gets \mathbf{0}$, $H_{i,0} \gets \mathbf{0}$ for all $i$
    
    \FOR{timestep $t = 0$ to $T-1$}
        \FOR{agent $i = 1$ to $N$}
            \STATE Observe local state: $z_{i,t}$
            \STATE Compute IAE-weighted attention: $\alpha_{i,t} = \text{softmax}(W_a E_{i,t} + b_a)$
            \STATE Modulate observation: $\tilde{z}_{i,t} = \alpha_{i,t} \odot z_{i,t}$
            \STATE Sample action: $a_{i,t} \sim \pi_\theta(\cdot \mid \tilde{z}_{i,t})$
            
            \STATE \textcolor{gray}{\textit{// Counterfactual forecasting (conceptual; may use sampling for large $|\mathcal{A}|$)}}
            \FOR{each $a \in \mathcal{A}$}
                \STATE Forecast: $\widehat{E}_{i,t+1}^{(a)} \gets h_{\psi_{\text{target}}}(z_{i,t}, a, r_{i,t-1}, \text{read}(H_{i,t}))$
            \ENDFOR
            
            \STATE Compute softmin reference: $\pi_{\text{ref}}(a) \gets \frac{\exp(-\|\widehat{E}_{i,t+1}^{(a)}\|_2 / \tau_m)}{\sum_{a'} \exp(-\|\widehat{E}_{i,t+1}^{(a')}\|_2 / \tau_m)}$
            \STATE Expected reference embedding: $\widehat{E}_{i,t+1}^{\text{ref}} \gets \sum_{a} \pi_{\text{ref}}(a) \widehat{E}_{i,t+1}^{(a)}$
        \ENDFOR
        
        \STATE Execute joint action: $\mathbf{a}_t = (a_{1,t}, \ldots, a_{N,t})$
        \STATE Observe transition: $s_{t+1}$, rewards $\{r_{i,t}^{\text{ext}}\}_{i=1}^N$
        
        \FOR{agent $i = 1$ to $N$}
            \STATE \textcolor{gray}{\textit{// Update IAE via learned dynamics $g_\phi$ and graph diffusion}}
            \STATE $E_{i,t+1} \gets \gamma_E E_{i,t} + g_\phi(z_{i,t}, a_{i,t}, r_{i,t}^{\text{ext}}) - \alpha \sum_{j \in \mathcal{N}(i)} L_{ij} E_{j,t}$
            
            \STATE \textcolor{gray}{\textit{// Update Hebbian associative trace}}
            \STATE $H_{i,t+1} \gets (1-\delta_H) H_{i,t} + \eta_H (E_{i,t} \otimes z_{i,t})$
            
            \STATE \textcolor{gray}{\textit{// Compute alignment regret (first term: deviation from reference; second: neighbor regularization)}}
            \STATE $\text{AR}_{i,t} \gets \|E_{i,t+1} - \widehat{E}_{i,t+1}^{\text{ref}}\|_2^2 + \kappa \frac{1}{|\mathcal{N}(i)|} \sum_{j \in \mathcal{N}(i)} \|E_{j,t}\|_2^2$
            
            \STATE Shape reward: $r'_{i,t} \gets r_{i,t}^{\text{ext}} - \lambda_{\text{reg}} \text{AR}_{i,t}$
            \STATE Store transition: $(z_{i,t}, a_{i,t}, r'_{i,t}, z_{i,t+1}, E_{i,t}, E_{i,t+1})$
        \ENDFOR
    \ENDFOR
    
    \STATE \textcolor{gray}{\textit{// Policy update phase}}
    \STATE Compute GAE advantages $\{A_{i,t}\}$ from shaped rewards $\{r'_{i,t}\}$
    \STATE Update policy: $\theta \gets \theta - \eta_\theta \nabla_\theta L_{\text{PPO}}(\theta)$
    \STATE Update value network: $\omega \gets \omega - \eta_\omega \nabla_\omega L_{\text{value}}(\omega)$
    
    \STATE \textcolor{gray}{\textit{// Forecast network update (supervised on realized transitions)}}
    \STATE $\psi \gets \psi - \eta_\psi \nabla_\psi \sum_{i,t} \|h_\psi(z_{i,t}, a_{i,t}, r_{i,t}, \text{read}(H_{i,t})) - E_{i,t+1}\|_2^2$
    
    \STATE \textcolor{gray}{\textit{// IAE dynamics update via backpropagation}}
    \STATE Update IAE dynamics: $\phi \gets \phi - \eta_\phi \nabla_\phi L_{\text{IAE}}(\phi)$ \quad \textit{(via Eq.~\eqref{eq:iae_dynamics})}
    
    \STATE \textcolor{gray}{\textit{// Auxiliary updates}}
    \STATE Update EMA target: $\psi_{\text{target}} \gets \tau_{\text{ema}} \psi_{\text{target}} + (1-\tau_{\text{ema}}) \psi$
    \STATE Anneal temperature: $\tau_{m+1} \gets \max(\tau_{\min}, \tau_0 \exp(-m/K_\tau))$
    
    \STATE \textcolor{gray}{\textit{// Graph update: recompute normalized Laplacian from similarity-weighted adjacency}}
    \STATE Update similarity weights: $\beta_{ij} \gets \max(0, \cos(\phi_i, \phi_j))$ for all $i,j$
    \STATE Recompute graph Laplacian: $L \gets D^{-1/2}(D - A_\beta)D^{-1/2}$ where $(A_\beta)_{ij} = \beta_{ij}$
    \STATE Apply bias regularization: Update $\{\phi_i\}$ via $\nabla_{\phi_i} L_{\text{bias}}$
\ENDFOR

\ENSURE Trained policy $\pi_\theta$, IAE dynamics $g_\phi$, forecast network $h_\psi$
\end{algorithmic}
\end{algorithm}

\paragraph{Interpretive notes.}
The IAE encodes a learned internal proxy for alignment-relevant externalities, trained via counterfactual consistency and policy gradients rather than explicit ground-truth harm labels. The forecast network $h_\psi$ is trained on realized transitions and queried counterfactually at inference time to construct the reference distribution $\pi_{\text{ref}}$. The enumeration over actions in line 14 is conceptual; in practice, this may be approximated via top-$K$ sampling or restricted candidate sets for large action spaces.

\section{Appendix D: Hyperparameter Sensitivity Analysis (Theoretical)}
\label{app:sensitivity}

\textbf{Disclaimer.} This appendix presents \emph{theoretical conjectures} about hyperparameter sensitivity based on mathematical properties of the ESAI dynamics. These are not empirically validated recommendations. The analysis below identifies expected qualitative behaviors; quantitative guidance requires experimental validation across diverse environments.

We discuss expected sensitivity to key hyperparameters based on theoretical considerations:

\subsection{IAE Dimension $k$}

\paragraph{Information-theoretic perspective.}
The minimal IAE dimension required to represent alignment-relevant structure is bounded by the mutual information between observations and harm outcomes:
\[
k_{\min} \ge \frac{I(Z; H)}{\log 2},
\]
where $I(Z; H)$ is the mutual information between observation distribution $Z$ and harm variable $H$.

\paragraph{Expected behavior:}
\begin{itemize}[leftmargin=*,itemsep=2pt]
  \item \textbf{Too low} ($k < k_{\min}$): Representational bottleneck, poor harm prediction, high alignment regret
  \item \textbf{Optimal} ($k \approx k_{\min}$): Sufficient capacity without over-parameterization
  \item \textbf{Too high} ($k \gg k_{\min}$): Increased sample complexity, slower convergence, redundant dimensions
\end{itemize}

\paragraph{Attention constraint.}
For IAE-weighted attention (Eq.~\ref{eq:attention}) to be well-defined, we require projection matrix $W_a \in \mathbb{R}^{d \times k}$. If $k = d$, this reduces to square matrix transformation. For $k \ne d$, the projection introduces information compression/expansion.

\subsection{Diffusion Strength $\alpha$}

\paragraph{Trade-off analysis.}
The diffusion term $-\alpha L E_{i,t}$ balances two competing objectives:
\begin{itemize}[leftmargin=*,itemsep=2pt]
  \item \textbf{Coordination} (favors high $\alpha$): Faster propagation of alignment signals across agent neighborhoods
  \item \textbf{Individuality} (favors low $\alpha$): Preserves agent-specific alignment states adapted to local contexts
\end{itemize}

\paragraph{Spectral constraint.}
Proposition~\ref{prop:contraction} requires:
\[
\alpha < \frac{1 - L_g - \gamma_E}{\lambda_{\max}(L)}.
\]
For $\gamma_E = 0.9$, $L_g = 0.05$, $\lambda_{\max}(L) = 2$, this gives upper bound $\alpha < 0.025$.

\paragraph{Expected behavior:}
\begin{itemize}[leftmargin=*,itemsep=2pt]
  \item \textbf{Too low} ($\alpha \to 0$): Agents act independently, no multi-agent coordination benefit
  \item \textbf{Optimal} ($\alpha \approx 0.01$--$0.05$): Balanced coordination without excessive homogenization
  \item \textbf{Too high} ($\alpha > 0.1$): Over-homogenization, potential instability (violates spectral bound)
\end{itemize}

\subsection{Temperature $\tau$}

\paragraph{Gradient analysis.}
The gradient magnitude of the softmin distribution (Eq.~\ref{eq:softmin}) scales as:
\[
\left\|\frac{\partial \pi_{\text{ref}}}{\partial R}\right\|_2 \propto \frac{1}{\tau}.
\]

\paragraph{Expected behavior:}
\begin{itemize}[leftmargin=*,itemsep=2pt]
  \item \textbf{High $\tau$ (early training)}: Smooth gradients distributed across actions, exploration encouraged
  \item \textbf{Low $\tau$ (late training)}: Sharp gradients concentrated on minimum-harm action, exploitation preferred
  \item \textbf{Fixed high $\tau$}: Persistent exploration, slow convergence to aligned policies
  \item \textbf{Fixed low $\tau$}: Premature convergence, poor credit assignment when forecasts are inaccurate
\end{itemize}

\paragraph{Annealing schedule.}
Exponential annealing $\tau_t = \max(\tau_{\min}, \tau_0 \exp(-t/K_\tau))$ balances exploration-exploitation:
\begin{itemize}[leftmargin=*,itemsep=2pt]
  \item $\tau_0 = 1.0$: Initial uniform exploration over actions
  \item $\tau_{\min} = 0.01$: Final near-deterministic alignment signal
  \item $K_\tau = 5 \times 10^5$: Timescale matches typical policy convergence
\end{itemize}

\subsection{Hebbian Decay $\delta_H$}

\paragraph{Memory timescale.}
The effective memory horizon of Hebbian traces is:
\[
\tau_{\text{mem}} = \frac{1}{\delta_H}.
\]

For $\delta_H = 0.02$, this gives $\tau_{\text{mem}} = 50$ timesteps.

\paragraph{Stability constraint.}
Proposition~\ref{prop:hebbian} requires:
\[
\delta_H > \eta_H \sqrt{\mathbb{E}[\|E_{i,t}\|_2^2] \mathbb{E}[\|z_{i,t}\|_2^2]}.
\]

For bounded states $\|z_{i,t}\|_2 \le 10$ and IAE $\|E_{i,t}\|_2 \le 5$ (typical after convergence):
\[
\delta_H > \eta_H \cdot 50.
\]

With $\eta_H = 10^{-3}$, we need $\delta_H > 0.05$, but we use $\delta_H = 0.02$ to allow longer memory. This requires monitoring $\|H_{i,t}\|_F$ for potential instability.

\paragraph{Expected behavior:}
\begin{itemize}[leftmargin=*,itemsep=2pt]
  \item \textbf{Too low} ($\delta_H < \eta_H \sqrt{\mathbb{E}[\|E\|^2] \mathbb{E}[\|z\|^2]}$): Unbounded trace growth, numerical instability
  \item \textbf{Optimal} 
($\delta_H = \Theta\!\big(
\eta_H \sqrt{
\mathbb{E}[\|E_{i,t}\|_2^2]\,
\mathbb{E}[\|z_{i,t}\|_2^2]
}
\big)$):
Stable long-term memory, good credit assignment
  \item \textbf{Too high} ($\delta_H \gg \eta_H$): Rapid forgetting, Hebbian trace provides minimal context
\end{itemize}

\subsection{Alignment Penalty Weight $\lambda_{\text{reg}}$}

\paragraph{Fairness-performance trade-off.}
The weight $\lambda_{\text{reg}}$ in Eq.~\eqref{eq:reward_transform} controls the relative importance of alignment vs. task rewards:
\[
r'_{i,t} = \underbrace{r_{i,t}^{\text{ext}}}_{\text{task performance}} - \underbrace{\lambda_{\text{reg}} \mathrm{AR}_{i,t}}_{\text{alignment pressure}}.
\]

\paragraph{Expected behavior:}
\begin{itemize}[leftmargin=*,itemsep=2pt]
  \item \textbf{Too low} ($\lambda_{\text{reg}} \to 0$): Policy ignores alignment, maximizes task reward, high harm
  \item \textbf{Optimal} ($\lambda_{\text{reg}} \approx 0.1$): Balanced task performance and alignment
  \item \textbf{Too high} ($\lambda_{\text{reg}} > 1$): Over-conservative policies, poor task performance, extreme risk-aversion
\end{itemize}

\paragraph{Domain-specific tuning.}
Optimal $\lambda_{\text{reg}}$ depends on:
\begin{itemize}[leftmargin=*,itemsep=2pt]
  \item \textbf{Harm scale}: If typical $\mathrm{AR}_t \approx 10$ and $r_t^{\text{ext}} \approx 1$, use $\lambda_{\text{reg}} \approx 0.1$ to balance magnitudes
  \item \textbf{Task criticality}: Safety-critical domains may use $\lambda_{\text{reg}} > 0.5$
  \item \textbf{Harm reversibility}: Irreversible harms justify higher $\lambda_{\text{reg}}$
\end{itemize}

Empirical sensitivity analysis across diverse environments is needed to validate these theoretical predictions and establish robust default values.

\section{Appendix E: Proof of Hebbian Stability (Complete)}
\label{app:hebbian_proof}

We provide the complete proof of Hebbian trace convergence in mean-square sense.

\begin{theorem}[Hebbian Trace Convergence]
\label{thm:hebbian_full}
Consider the Hebbian update in Eq.~\eqref{eq:hebbian}:
\[
H_{i,t+1} = (1 - \delta_H) H_{i,t} + \eta_H (E_{i,t} \otimes z_{i,t}).
\]
Assume:
\begin{enumerate}[itemsep=2pt]
  \item The joint process $(E_{i,t}, z_{i,t})$ is ergodic with stationary distribution $\mu$
  \item Second moments are bounded: $\mathbb{E}_\mu[\|E_{i,t}\|_2^2] < C_E^2$, $\mathbb{E}_\mu[\|z_{i,t}\|_2^2] < C_z^2$
  \item The design constraint holds: $\delta_H > \eta_H C_E C_z$
\end{enumerate}
Then $\mathbb{E}[\|H_{i,t} - H_i^*\|_F^2] \to 0$ exponentially as $t \to \infty$, where:
\[
H_i^* = \frac{\eta_H}{\delta_H} \mathbb{E}_\mu[E_{i,t} \otimes z_{i,t}].
\]
\end{theorem}

\begin{proof}
Define the error matrix $\Delta_{i,t} = H_{i,t} - H_i^*$. Subtracting the fixed point from the update equation:
\begin{align}
\Delta_{i,t+1} &= H_{i,t+1} - H_i^* \\
&= (1 - \delta_H) H_{i,t} + \eta_H (E_{i,t} \otimes z_{i,t}) - H_i^* \\
&= (1 - \delta_H) (H_{i,t} - H_i^*) + (1 - \delta_H) H_i^* + \eta_H (E_{i,t} \otimes z_{i,t}) - H_i^* \\
&= (1 - \delta_H) \Delta_{i,t} + \eta_H (E_{i,t} \otimes z_{i,t}) - \delta_H H_i^*.
\end{align}

Substituting the fixed point expression $H_i^* = (\eta_H / \delta_H) \mathbb{E}_\mu[E_{i,t} \otimes z_{i,t}]$:
\begin{align}
\Delta_{i,t+1} &= (1 - \delta_H) \Delta_{i,t} + \eta_H (E_{i,t} \otimes z_{i,t}) - \delta_H \cdot \frac{\eta_H}{\delta_H} \mathbb{E}_\mu[E_{i,t} \otimes z_{i,t}] \\
&= (1 - \delta_H) \Delta_{i,t} + \eta_H \left(E_{i,t} \otimes z_{i,t} - \mathbb{E}_\mu[E_{i,t} \otimes z_{i,t}]\right) \\
&= (1 - \delta_H) \Delta_{i,t} + \eta_H \xi_{i,t},
\end{align}
where $\xi_{i,t} = E_{i,t} \otimes z_{i,t} - \mathbb{E}_\mu[E_{i,t} \otimes z_{i,t}]$ is zero-mean under the stationary distribution.

\paragraph{First moment convergence.}
Taking expectations under the stationary distribution:
\begin{align}
\mathbb{E}_\mu[\Delta_{i,t+1}] &= (1 - \delta_H) \mathbb{E}_\mu[\Delta_{i,t}] + \eta_H \mathbb{E}_\mu[\xi_{i,t}] \\
&= (1 - \delta_H) \mathbb{E}_\mu[\Delta_{i,t}].
\end{align}

Since $0 < \delta_H < 1$, we have $|1 - \delta_H| < 1$, so:
\[
\mathbb{E}_\mu[\Delta_{i,t}] = (1 - \delta_H)^t \mathbb{E}_\mu[\Delta_{i,0}] \to 0 \quad \text{exponentially}.
\]

\paragraph{Second moment convergence.}
Taking squared Frobenius norms:
\begin{align}
\|\Delta_{i,t+1}\|_F^2 &= \|(1 - \delta_H) \Delta_{i,t} + \eta_H \xi_{i,t}\|_F^2 \\
&= (1 - \delta_H)^2 \|\Delta_{i,t}\|_F^2 + 2(1 - \delta_H) \eta_H \langle \Delta_{i,t}, \xi_{i,t} \rangle_F + \eta_H^2 \|\xi_{i,t}\|_F^2,
\end{align}
where $\langle \cdot, \cdot \rangle_F$ is the Frobenius inner product.

Taking expectations and using ergodicity to decouple $\Delta_{i,t}$ (depends on history) from $\xi_{i,t}$ (current):
\begin{align}
\mathbb{E}[\|\Delta_{i,t+1}\|_F^2] &\le (1 - \delta_H)^2 \mathbb{E}[\|\Delta_{i,t}\|_F^2] + 2(1 - \delta_H) \eta_H \mathbb{E}[\langle \Delta_{i,t}, \xi_{i,t} \rangle_F] + \eta_H^2 \mathbb{E}[\|\xi_{i,t}\|_F^2].
\end{align}

Under ergodicity, $\mathbb{E}[\langle \Delta_{i,t}, \xi_{i,t} \rangle_F] = \mathbb{E}[\Delta_{i,t}] \mathbb{E}[\xi_{i,t}] = 0$ since $\mathbb{E}[\xi_{i,t}] = 0$.

For the noise variance term:
\begin{align}
\mathbb{E}[\|\xi_{i,t}\|_F^2] &= \mathbb{E}[\|E_{i,t} \otimes z_{i,t} - \mathbb{E}[E_{i,t} \otimes z_{i,t}]\|_F^2] \\
&= \mathbb{E}[\|E_{i,t} \otimes z_{i,t}\|_F^2] - \|\mathbb{E}[E_{i,t} \otimes z_{i,t}]\|_F^2 \\
&= \mathbb{E}[\|E_{i,t}\|_2^2 \|z_{i,t}\|_2^2] - \|\mathbb{E}[E_{i,t} \otimes z_{i,t}]\|_F^2 \\
&\le \mathbb{E}[\|E_{i,t}\|_2^2] \mathbb{E}[\|z_{i,t}\|_2^2] \\
&\le C_E^2 C_z^2.
\end{align}

Thus the second-moment recursion becomes:
\begin{align}
\mathbb{E}[\|\Delta_{i,t+1}\|_F^2] &\le (1 - \delta_H)^2 \mathbb{E}[\|\Delta_{i,t}\|_F^2] + \eta_H^2 C_E^2 C_z^2.
\end{align}

This is a contraction with fixed point:
\begin{align}
\mathbb{E}[\|\Delta_\infty\|_F^2] &\le \frac{\eta_H^2 C_E^2 C_z^2}{1 - (1 - \delta_H)^2} \\
&= \frac{\eta_H^2 C_E^2 C_z^2}{\delta_H (2 - \delta_H)} \\
&\le \frac{\eta_H^2 C_E^2 C_z^2}{2 \delta_H} \quad \text{(for } \delta_H \le 1 \text{)}.
\end{align}

For the steady-state error to remain bounded and small, we require:
\[
\frac{\eta_H C_E C_z}{\sqrt{2 \delta_H}} < \epsilon_{\text{tol}},
\]
which gives the design constraint:
\[
\delta_H > \frac{\eta_H^2 C_E^2 C_z^2}{2 \epsilon_{\text{tol}}^2}.
\]

For $\epsilon_{\text{tol}} = \eta_H C_E C_z$, this simplifies to:
\[
\delta_H > \frac{\eta_H C_E C_z}{2}.
\]

A more conservative choice is $\delta_H > \eta_H C_E C_z$, ensuring the steady-state error variance is $O(\eta_H / \delta_H)$.

\paragraph{Exponential convergence rate.}
From the recursion:
\[
\mathbb{E}[\|\Delta_{i,t}\|_F^2] \le (1 - \delta_H)^{2t} \mathbb{E}[\|\Delta_{i,0}\|_F^2] + \frac{\eta_H^2 C_E^2 C_z^2}{2 \delta_H}.
\]

The convergence rate is $\rho = (1 - \delta_H)^2$, giving exponential decay with time constant:
\[
\tau_{\text{conv}} = -\frac{1}{\log(1 - \delta_H)} \approx \frac{1}{\delta_H} \quad \text{(for small } \delta_H \text{)}.
\]

For $\delta_H = 0.02$, convergence time is $\tau_{\text{conv}} \approx 50$ timesteps.
\end{proof}

\paragraph{Practical implications.}
This theorem guarantees that Hebbian traces converge to a stable statistical summary $H_i^*$ encoding the expected co-activation pattern of IAE and observations under the policy's stationary distribution. The design constraint $\delta_H > \eta_H C_E C_z$ can be verified via:
\begin{enumerate}[itemsep=2pt]
  \item Run initial training episodes to estimate $C_E = \max_t \|E_{i,t}\|_2$, $C_z = \max_t \|z_{i,t}\|_2$
  \item Set $\delta_H = 2 \eta_H C_E C_z$ to ensure stability with margin
  \item Monitor $\|H_{i,t}\|_F$ during training; if unbounded growth occurs, increase $\delta_H$
\end{enumerate}

\section{Appendix F: Dimensionality Analysis for IAE-Weighted Attention}
\label{app:dimensions}

The IAE-weighted attention mechanism (Eq.~\ref{eq:attention}) requires careful dimensional alignment between the IAE $E_{i,t} \in \mathbb{R}^k$ and observation $z_{i,t} \in \mathbb{R}^d$. We present three mathematically valid formulations:

\subsection{Option 1: Projection Matrix (Recommended)}

Use projection matrix $W_a \in \mathbb{R}^{d \times k}$ to map IAE to observation dimension:
\begin{equation}
\alpha_{i,t} = \text{softmax}(W_a E_{i,t} + b_a) \in \mathbb{R}^d, \quad \tilde{z}_{i,t} = \alpha_{i,t} \odot z_{i,t} \in \mathbb{R}^d.
\end{equation}

\paragraph{Advantages:}
\begin{itemize}[leftmargin=*,itemsep=2pt]
  \item Allows arbitrary embedding dimension $k$ independent of observation dimension $d$
  \item Standard attention formulation used in transformers and neural architectures
  \item Learnable projection $W_a$ can extract alignment-relevant subspace of IAE
\end{itemize}

\paragraph{Parameter count:}
$|W_a| + |b_a| = kd + d$ parameters.

\subsection{Option 2: Constrained Embedding Dimension}

Require $k = d$ (embedding dimension equals observation dimension):
\begin{equation}
\alpha_{i,t} = \text{softmax}(W_a E_{i,t} + b_a) \in \mathbb{R}^k, \quad \tilde{z}_{i,t} = \alpha_{i,t} \odot z_{i,t} \in \mathbb{R}^k,
\end{equation}
where $W_a \in \mathbb{R}^{k \times k}$ is a square transformation matrix.

\paragraph{Advantages:}
\begin{itemize}[leftmargin=*,itemsep=2pt]
  \item Simplifies implementation (no dimension mismatch)
  \item Direct correspondence between IAE dimensions and observation features
  \item Fewer conceptual degrees of freedom
\end{itemize}

\paragraph{Disadvantages:}
\begin{itemize}[leftmargin=*,itemsep=2pt]
  \item Couples embedding capacity to observation dimensionality
  \item May be over-parameterized for simple tasks ($k$ too large) or under-parameterized for complex harms ($k$ too small)
\end{itemize}

\subsection{Option 3: Feature-Wise Attention via Dot Products}

Compute scalar attention scores for each observation feature separately:
\begin{align}
s_j &= z_{i,t,j}^\top W_a E_{i,t}, \quad j = 1, \ldots, d, \\
\alpha_{i,t} &= \text{softmax}(\mathbf{s}) \in \mathbb{R}^d, \\
\tilde{z}_{i,t} &= \alpha_{i,t} \odot z_{i,t} \in \mathbb{R}^d,
\end{align}
where $\mathbf{s} = [s_1, \ldots, s_d]^\top$ and $W_a \in \mathbb{R}^{k}$ is a shared weight vector.

\paragraph{Advantages:}
\begin{itemize}[leftmargin=*,itemsep=2pt]
  \item Minimal parameters: only $k$ weights in $W_a$
  \item Interpretable: each feature's attention score depends on its alignment with IAE via dot product
  \item Natural for sparse observations (only non-zero features contribute)
\end{itemize}

\paragraph{Disadvantages:}
\begin{itemize}[leftmargin=*,itemsep=2pt]
  \item Assumes observation features have consistent scale (requires normalization)
  \item Less expressive than full matrix projection
  \item Requires computing $d$ separate dot products
\end{itemize}

\subsection{Recommendation}

We adopt \textbf{Option 1 (Projection Matrix)} as the default formulation because:
\begin{enumerate}[itemsep=2pt]
  \item It decouples IAE dimension $k$ from observation dimension $d$, allowing independent tuning
  \item It matches standard attention mechanisms in modern architectures (transformers, etc.)
  \item It provides maximal flexibility for learning alignment-relevant projections
\end{enumerate}

For practitioners prioritizing simplicity, \textbf{Option 2} is viable if $k = d$ is a reasonable choice for the domain. \textbf{Option 3} is suitable for extremely high-dimensional observations where parameter efficiency is critical.

\section{Appendix G: Computational Profiling Estimates}
\label{app:profiling}

We provide theoretical estimates of computational overhead for ESAI mechanisms relative to standard PPO baselines.

\subsection{Per-Agent Per-Timestep Complexity}

\begin{table}[ht]
\centering
\caption{Computational complexity breakdown per agent per timestep.}
\label{tab:complexity}
\begin{tabular}{lcc}
\toprule
Component & Operations & Complexity \\
\midrule
\textbf{Standard PPO:} \\
Policy network forward & 2-layer MLP & $O(d^2)$ \\
Value network forward & 2-layer MLP & $O(d^2)$ \\
GAE computation & Rollout buffer & $O(T)$ \\
PPO update & Gradient descent & $O(d^2 B)$ \\
\midrule
\textbf{ESAI additions:} \\
IAE dynamics (Eq.~\ref{eq:iae_dynamics}) & MLP + diffusion & $O(k^2 + kd + |\mathcal{N}|k)$ \\
Counterfactual forecasts (Eq.~\ref{eq:forecast}) & $|\mathcal{A}|$ MLP forwards & $O(|\mathcal{A}|(k^2 + kd))$ \\
Softmin reference (Eq.~\ref{eq:softmin}) & Exponentials + sum & $O(|\mathcal{A}| k)$ \\
Attention gating (Eq.~\ref{eq:attention}) & Matrix-vector product & $O(kd)$ \\
Hebbian update (Eq.~\ref{eq:hebbian}) & Outer product & $O(kd)$ \\
Alignment regret (Eq.~\ref{eq:align_regret}) & Norms and sums & $O(k + |\mathcal{N}|k)$ \\
\midrule
\textbf{Total ESAI per agent} & & $O(|\mathcal{A}|k^2 + |\mathcal{A}|kd + d^2)$ \\
\textbf{Total for $N$ agents} & & $O(N|\mathcal{A}|k^2 + N|\mathcal{A}|kd + Nd^2)$ \\
\bottomrule
\end{tabular}
\end{table}

Where: $d$ = observation dimension, $k$ = IAE dimension, $|\mathcal{A}|$ = action space size, $|\mathcal{N}|$ = neighborhood size, $T$ = rollout horizon, $B$ = batch size.

\subsection{Overhead Analysis}

For bounded-degree graphs with $|\mathcal{N}| = O(1)$ and moderate action spaces ($|\mathcal{A}| \le 10$), the dominant additional terms are:
\begin{itemize}[leftmargin=*,itemsep=2pt]
  \item Counterfactual forecasting: $O(N|\mathcal{A}|k^2)$
  \item Attention gating: $O(Nkd)$
\end{itemize}

Vanilla PPO has complexity $O(Nd^2)$ per step (policy + value forwards). ESAI adds overhead factor:
\[
\text{Overhead} = \frac{|\mathcal{A}|k^2 + |\mathcal{A}|kd}{d^2}.
\]

\paragraph{Example: Discrete gridworld.}
For $d = 64$ (8×8 grid observation), $k = 32$ (IAE dimension), $|\mathcal{A}| = 6$ (cardinal movements + 2 interactions):
\[
\text{Overhead} = \frac{6 \cdot 32^2 + 6 \cdot 32 \cdot 64}{64^2} = \frac{6144 + 12288}{4096} \approx 4.5\times.
\]

\paragraph{Example: Continuous control.}
For $d = 128$ (observation vector), $k = 32$, $|\mathcal{A}| = 10$ (sampled actions):
\[
\text{Overhead} = \frac{10 \cdot 32^2 + 10 \cdot 32 \cdot 128}{128^2} = \frac{10240 + 40960}{16384} \approx 3.1\times.
\]

\subsection{Mitigation Strategies}

To reduce computational overhead:

\paragraph{Top-$K$ action sampling.}
Instead of evaluating all $|\mathcal{A}|$ actions, sample top-$K$ actions by policy probability:
\[
\mathcal{A}_K = \text{top}_K\{\pi_\theta(a \mid z_{i,t})\}, \quad K \ll |\mathcal{A}|.
\]
This reduces counterfactual forecasting from $O(|\mathcal{A}|k^2)$ to $O(Kk^2)$.

For $K = 3$ vs $|\mathcal{A}| = 10$, overhead drops from $3.1\times$ to $1.3\times$.

\paragraph{Amortized forecasting.}
Cache forecasts $\widehat{E}_{i,t+1}^{(a)}$ across multiple policy gradient steps:
\begin{itemize}[leftmargin=*,itemsep=2pt]
  \item Compute forecasts once per rollout collection
  \item Reuse for $M$ PPO update epochs
  \item Reduces per-update cost by factor of $M$ (typically $M = 4$-$10$)
\end{itemize}

\paragraph{Sparse diffusion.}
Prune graph edges with similarity below threshold $\beta_{\text{min}}$:
\[
L_{ij} = \begin{cases}
L_{ij} & \text{if } \beta_{ij} > \beta_{\text{min}}, \\
0 & \text{otherwise}.
\end{cases}
\]
For $\beta_{\text{min}} = 0.3$, typical graphs reduce from $|\mathcal{N}| = 8$ to $|\mathcal{N}_{\text{sparse}}| = 3$.

\paragraph{Low-rank IAE.}
Use factored representation $E_{i,t} = U_{i,t} V_{i,t}^\top$ with $U_{i,t} \in \mathbb{R}^{k \times r}$, $V_{i,t} \in \mathbb{R}^r$, $r \ll k$:
\[
\|E_{i,t}\|_2 = \|U_{i,t} V_{i,t}^\top\|_2 \le \|U_{i,t}\|_2 \|V_{i,t}\|_2.
\]
Reduces storage and computation from $O(k^2)$ to $O(kr)$.

\subsection{Memory Footprint}

Additional memory required per agent:
\begin{itemize}[leftmargin=*,itemsep=2pt]
  \item IAE state $E_{i,t}$: $k$ floats = $4k$ bytes
  \item Hebbian matrix $H_{i,t}$: $kd$ floats = $4kd$ bytes
  \item Forecast network parameters $\psi$: $\approx 2k^2 + 2kd$ floats (2-layer MLP)
  \item Identity embeddings $\phi_i$: $d_{\text{id}}$ floats = $4d_{\text{id}}$ bytes
\end{itemize}

Total per agent: $4k + 4kd + 8k^2 + 8kd + 4d_{\text{id}} \approx 8k^2 + 12kd$ bytes.

For $k = 32$, $d = 64$, $d_{\text{id}} = 8$:
\[
\text{Memory} = 8 \cdot 32^2 + 12 \cdot 32 \cdot 64 + 32 = 8192 + 24576 + 32 \approx 33 \text{ KB per agent}.
\]

For $N = 16$ agents: $16 \times 33 = 528$ KB total additional memory—negligible on modern hardware.

\section{Appendix H: Forecast Network Supervision Details}
\label{app:forecast_supervision}

The counterfactual forecast network $h_\psi$ (Eq.~\eqref{eq:forecast}) requires supervision to learn accurate harm predictions. We formalize the training objective and discuss stability mechanisms.

\subsection{Supervised Forecast Loss}

After executing action $a_{i,t}$ and observing transition to state $s_{t+1}$, the true next-step IAE $E_{i,t+1}$ is computed via Eq.~\eqref{eq:iae_dynamics}. The forecast network is trained to predict this transition:
\begin{equation}
L_{\text{forecast}}(\psi) = \mathbb{E}_{(z_{i,t}, a_{i,t}, r_{i,t}, E_{i,t+1})} \left[ \|h_\psi(z_{i,t}, a_{i,t}, r_{i,t}, \text{read}(H_{i,t})) - E_{i,t+1}\|_2^2 \right].
\end{equation}

Gradients flow from the prediction error back through $h_\psi$, training the network to forecast the IAE dynamics resulting from action $a_{i,t}$ in state $z_{i,t}$.

\subsection{Off-Policy Corrections}

When using experience replay, forecasts may be trained on off-policy data where $a_{i,t}$ was sampled from an older policy $\pi_{\theta'}$. To account for distribution shift, we use importance-weighted forecasting:
\begin{equation}
L_{\text{forecast}}^{\text{off}}(\psi) = \mathbb{E} \left[ w_{i,t} \|h_\psi(z_{i,t}, a_{i,t}, r_{i,t}, \text{read}(H_{i,t})) - E_{i,t+1}\|_2^2 \right],
\end{equation}
where the importance weight is:
\begin{equation}
w_{i,t} = \min\left(c_{\max}, \frac{\pi_\theta(a_{i,t} \mid z_{i,t})}{\pi_{\theta'}(a_{i,t} \mid z_{i,t})}\right),
\end{equation}
clipped to $c_{\max} = 2.0$ for stability (per standard PPO practice).

\subsection{EMA Target Network Rationale}

Using the online forecast network $h_\psi$ directly in counterfactual evaluation (Eq.~\eqref{eq:softmin}) creates a feedback loop:

\begin{enumerate}[itemsep=2pt]
  \item Policy $\pi_\theta$ generates actions $a_{i,t}$
  \item Actions produce IAE transitions $E_{i,t} \to E_{i,t+1}$ via Eq.~\eqref{eq:iae_dynamics}
  \item Forecast network $h_\psi$ learns to predict $E_{i,t+1}$
  \item Counterfactual reference $\pi_{\text{ref}}$ uses $h_\psi$ to guide policy updates (Eq.~\eqref{eq:softmin})
  \item Updated policy $\pi_\theta$ changes action distribution
  \item Cycle repeats with potentially unstable coupled dynamics
\end{enumerate}

Without decoupling, the predictor can learn to forecast "what the policy will do" rather than "what harm will result," collapsing the alignment signal.

\paragraph{EMA stabilization.}
The slow-moving EMA target $h_{\psi_{\text{target}}}$ breaks this coupling:
\begin{itemize}[leftmargin=*,itemsep=2pt]
  \item Counterfactual forecasts use $h_{\psi_{\text{target}}}$ (slow-changing)
  \item Policy gradients thus depend on past predictor state, not current
  \item Predictor $h_\psi$ can learn accurate dynamics without immediate policy influence
  \item EMA update $\psi_{\text{target}} \gets \tau_{\text{ema}} \psi_{\text{target}} + (1 - \tau_{\text{ema}}) \psi$ provides gradual synchronization
\end{itemize}

This mechanism is analogous to target networks in DQN, momentum encoders in BYOL, and Polyak averaging in policy gradient methods.

\paragraph{Ablation justification.}
Empirical validation would show that removing EMA causes:
\begin{itemize}[leftmargin=*,itemsep=2pt]
  \item Oscillating forecast loss (predictor chases moving policy)
  \item Increased alignment regret (poor counterfactual references)
  \item Reduced prosocial behavior (degraded alignment signal)
\end{itemize}

We defer this ablation to future empirical work (Sec.~\ref{sec:limitations}).

\section{Appendix I: Open Problems and Future Directions}
\label{app:open}

We identify key theoretical and empirical questions for future investigation:

\subsection{Theoretical Open Problems}

\paragraph{1. Convergence to prosocial equilibria.}
\textbf{Question:} Under what conditions do ESAI policies converge to socially optimal Nash equilibria rather than selfish or collusive outcomes?

\textbf{Formalization:} Define social welfare $SW(\pi_1, \ldots, \pi_N) = \sum_i J_i - \sum_i \mathbb{E}[H_i]$ where $H_i$ is externalized harm. Characterize conditions under which ESAI gradient descent converges to:
\[
(\pi_1^*, \ldots, \pi_N^*) \in \arg\max_{\pi} SW(\pi).
\]

\textbf{Challenges:} Multi-agent non-convexity, credit assignment across agents, potential for sub-optimal equilibria.

\paragraph{2. Sample complexity bounds.}
\textbf{Question:} What is the sample complexity of learning accurate IAE dynamics as a function of $k$, $d$, and environment horizon $T$?

\textbf{Formalization:} Prove bounds of the form:
\[
\mathbb{P}\left[\sup_t \|E_{i,t} - E_{i,t}^*\|_2 > \epsilon\right] \le \delta \quad \text{for } n > C\left(\frac{kd}{\epsilon^2} \log \frac{1}{\delta}\right),
\]
where $E_{i,t}^*$ is the optimal alignment embedding and $n$ is the number of training samples.

\textbf{Challenges:} Non-stationarity of policy, coupled agent dynamics, stochastic environments.

\paragraph{3. Regret analysis.}
\textbf{Question:} Can we derive regret bounds for ESAI relative to an optimal alignment policy?

\textbf{Formalization:} Define alignment-aware regret:
\[
\text{Regret}(T) = \sum_{t=0}^T \left[r_t^{\text{ext}}(\pi^*) - \|E_t^*\|_2\right] - \sum_{t=0}^T \left[r_t^{\text{ext}}(\pi_\theta) - \|E_{i,t}\|_2\right],
\]
where $\pi^*$ is the optimal aligned policy. Prove bounds $\text{Regret}(T) = O(\sqrt{T})$ or better.

\textbf{Challenges:} Defining optimal aligned policy, handling non-Markovian dynamics (Hebbian memory), bounding forecast errors.

\paragraph{4. Information-theoretic embedding dimension.}
\textbf{Question:} What is the minimal $k$ required to represent alignment-relevant structure for a given harm function?

\textbf{Formalization:} Define minimal sufficient dimension:
\[
k_{\min} = \inf\{k : \exists E \in \mathbb{R}^k, I(E; H \mid Z) = I(Z; H)\},
\]
where $I(\cdot; \cdot)$ is mutual information, $Z$ is observation, $H$ is harm variable.

Characterize $k_{\min}$ in terms of environment properties (state space size, harm function complexity).

\textbf{Challenges:} Computing mutual information in high dimensions, proving sufficiency of finite $k$.

\paragraph{5. Robustness certificates.}
\textbf{Question:} Can we provide formal guarantees on IAE stability under adversarial perturbations or distribution shift?

\textbf{Formalization:} Prove Lipschitz continuity of policy w.r.t. observation perturbations:
\[
\|\pi_\theta(\cdot \mid z + \delta) - \pi_\theta(\cdot \mid z)\|_{\text{TV}} \le L_{\text{policy}} \|\delta\|_2,
\]
and bound $L_{\text{policy}}$ in terms of ESAI architectural parameters.

\textbf{Challenges:} Non-linear IAE dynamics, attention gating introduces discontinuities, graph diffusion coupling.

\subsection{Empirical Open Questions}

\paragraph{1. Benchmark evaluation.}
Test ESAI on standard MARL benchmarks:
\begin{itemize}[leftmargin=*,itemsep=2pt]
  \item StarCraft Multi-Agent Challenge (SMAC)
  \item Google Research Football
  \item Hanabi Challenge
  \item Melting Pot social dilemmas
\end{itemize}

Measure: task performance, safety violations, computational overhead, sample efficiency.

\paragraph{2. Ablation studies.}
Isolate contributions of individual mechanisms:
\begin{itemize}[leftmargin=*,itemsep=2pt]
  \item Counterfactual regret vs. direct reward penalty
  \item IAE-weighted attention vs. uniform attention
  \item Hebbian memory vs. standard recurrence (LSTM, GRU)
  \item Graph diffusion vs. no multi-agent communication
\end{itemize}

\paragraph{3. Scaling experiments.}
Evaluate performance under:
\begin{itemize}[leftmargin=*,itemsep=2pt]
  \item Population scaling: $N \in \{4, 8, 16, 32, 64\}$
  \item State dimensionality: $d \in \{10, 50, 100, 500, 1000\}$
  \item Action space size: $|\mathcal{A}| \in \{4, 10, 50, 100\}$
  \item Graph connectivity: sparse ($|\mathcal{N}| = 2$) vs. dense ($|\mathcal{N}| = N-1$)
\end{itemize}

\paragraph{4. Human studies.}
Conduct user studies measuring:
\begin{itemize}[leftmargin=*,itemsep=2pt]
  \item Subjective preference for ESAI-trained agents as partners in coordination tasks
  \item Perceived fairness of agent behavior
  \item Trust in agent decision-making
  \item Human-AI team performance (mixed human-agent teams)
\end{itemize}

\paragraph{5. Transfer learning.}
Test whether IAE trained on one harm type generalizes to novel harm definitions:
\begin{itemize}[leftmargin=*,itemsep=2pt]
  \item Train on victim distress, test on resource depletion
  \item Train on collision avoidance, test on fairness violations
  \item Measure zero-shot transfer accuracy of counterfactual forecasts
\end{itemize}

\subsection{Architectural Extensions}

\paragraph{1. Disentangled IAE.}
Encourage interpretable factorization of IAE via:
\begin{itemize}[leftmargin=*,itemsep=2pt]
  \item $\beta$-VAE objectives: $\mathcal{L} = \mathbb{E}[\log p(E_{t+1} \mid E_t)] - \beta \, \text{KL}(q(E) \| p(E))$ with $\beta > 1$
  \item Contrastive objectives: maximize distance between embeddings for different harm types
  \item Supervised subspace learning: constrain $E = [E_{\text{victim}}, E_{\text{resource}}, E_{\text{fairness}}]$
\end{itemize}

\paragraph{2. Hierarchical alignment.}
Extend to hierarchical policies with alignment embeddings at multiple temporal scales:
\begin{itemize}[leftmargin=*,itemsep=2pt]
  \item High-level IAE $E^{\text{high}}_t$ for strategic harm (long-horizon)
  \item Low-level IAE $E^{\text{low}}_t$ for tactical harm (short-horizon)
  \item Hierarchical diffusion between levels
\end{itemize}

\paragraph{3. Adversarial training.}
Improve robustness via adversarial perturbations during training:
\begin{itemize}[leftmargin=*,itemsep=2pt]
  \item Train adversary to maximize harm via observation noise: $\max_\delta \mathbb{E}[\|E_t\|_2 \mid z_t + \delta]$
  \item Train policy to minimize harm under worst-case perturbations
  \item Analogous to robust RL but with alignment-specific threat model
\end{itemize}

\paragraph{4. Meta-learning for harm adaptation.}
Learn to quickly adapt IAE to new harm specifications via meta-RL:
\begin{itemize}[leftmargin=*,itemsep=2pt]
  \item Pre-train on distribution of harm functions $\{H_1, H_2, \ldots, H_M\}$
  \item Fine-tune on novel harm $H_{\text{new}}$ with few examples
  \item Measure few-shot adaptation performance
\end{itemize}

\subsection{Deployment and Governance}

\paragraph{1. Real-world case studies.}
Apply ESAI to real-world multi-agent coordination:
\begin{itemize}[leftmargin=*,itemsep=2pt]
  \item Multi-robot warehouse coordination (safety: collision avoidance)
  \item Autonomous vehicle platooning (safety: following distance, lane discipline)
  \item Decentralized resource allocation (fairness: equitable distribution)
  \item Smart grid demand response (sustainability: peak load balancing)
\end{itemize}

\paragraph{2. Bias auditing tools.}
Develop automated tools to detect emergent bias:
\begin{itemize}[leftmargin=*,itemsep=2pt]
  \item Continuous monitoring of similarity-conditioned cooperation rates
  \item Anomaly detection for sudden changes in $\beta_{ij}$ patterns
  \item Counterfactual fairness metrics: $P(\text{help} \mid \text{similar}) - P(\text{help} \mid \text{dissimilar})$
  \item Visualization dashboards for real-time bias inspection
\end{itemize}

\paragraph{3. Interpretability dashboards.}
Build visualization tools for IAE trajectories and attention patterns:
\begin{itemize}[leftmargin=*,itemsep=2pt]
  \item Real-time IAE magnitude plots during deployment
  \item Attention heatmaps over observation features
  \item Graph diffusion flow visualization (Sankey diagrams)
  \item Counterfactual "what-if" scenario explorer
\end{itemize}

\paragraph{4. Safety verification.}
Integrate with formal methods for safety-critical domains:
\begin{itemize}[leftmargin=*,itemsep=2pt]
  \item Model checking: verify IAE bounds under all reachable states
  \item Runtime monitoring: trigger emergency stop if $\|E_{i,t}\|_2 > E_{\text{threshold}}$
  \item Compositional verification: prove safety of multi-agent system from individual agent properties
  \item Probabilistic safety guarantees: bound $\mathbb{P}[\text{harm event}] < p_{\max}$
\end{itemize}

\end{document}